\documentclass[10pt,twocolumn,letterpaper]{article}

\usepackage{iccv}
\usepackage{times}
\usepackage{epsfig}
\usepackage{graphicx}
\usepackage{amsmath}
\usepackage{amssymb}
\usepackage{booktabs}
\usepackage{caption}
\usepackage{color}
\usepackage{colortbl}
\usepackage{xcolor} 
\usepackage{caption}
\usepackage{subcaption}
\usepackage{longtable}
\usepackage{float}
\usepackage{footnote}
\usepackage{url}
\usepackage[flushleft]{threeparttable}
\usepackage{enumitem}
\usepackage{algorithm}
\usepackage{algorithmic}
\usepackage{listings}
\usepackage[pagebackref,breaklinks,colorlinks,bookmarks=false]{hyperref}
\hypersetup{citecolor=[RGB]{119,185,0}}
\usepackage{diagbox}
\usepackage{makecell} 
\usepackage{framed}
\usepackage{makecell}
\usepackage{listings}
\definecolor{lightgray}{rgb}{.9,.9,.9}
\definecolor{darkgray}{rgb}{.4,.4,.4}
\definecolor{purple}{rgb}{0.65, 0.12, 0.82}

\iccvfinalcopy

\usepackage[mathscr]{eucal}
\usepackage{amssymb,amsmath,amsfonts,latexsym}
\usepackage{amsmath,graphicx,bm,xcolor,url}
\usepackage{bm}
\usepackage{mathrsfs}
 \usepackage{amsthm}

\catcode`~=11 \def\UrlSpecials{\do\~{\kern -.15em\lower .7ex\hbox{~}\kern .04em}} \catcode`~=13 

\allowdisplaybreaks[3]

\newcommand{\calC}{\mathcal{C}}

\newcommand{\calN}{\mathcal{N}}

\newcommand{\calS}{\mathcal{S}}

\newcommand{\calX}{\mathcal{X}}

\newcommand{\ba}{\mathbf{a}}
\newcommand{\bA}{\mathbf{A}}
\newcommand{\bb}{\mathbf{b}}
\newcommand{\bB}{\mathbf{B}}

\newcommand{\bE}{\mathbf{E}}

\newcommand{\bH}{\mathbf{H}}

\newcommand{\bI}{\mathbf{I}}

\newcommand{\bs}{\mathbf{s}}

\newcommand{\bw}{\mathbf{w}}
\newcommand{\bW}{\mathbf{W}}
\newcommand{\bx}{\mathbf{x}}
\newcommand{\bX}{\mathbf{X}}
\newcommand{\by}{\mathbf{y}}

\newcommand{\bz}{\mathbf{z}}
\newcommand{\bZ}{\mathbf{Z}}

\newcommand{\rmd}{\mathrm{d}}

\newcommand{\bbE}{\mathbb{E}}

\newcommand{\bbP}{\mathbb{P}}

\newcommand{\bbR}{\mathbb{R}}

\DeclareMathAlphabet{\mathbsf}{OT1}{cmss}{bx}{n}
\DeclareMathAlphabet{\mathssf}{OT1}{cmss}{m}{sl}

\DeclareSymbolFont{bsfletters}{OT1}{cmss}{bx}{n}  
\DeclareSymbolFont{ssfletters}{OT1}{cmss}{m}{n}
\DeclareMathSymbol{\bsfGamma}{0}{bsfletters}{'000}
\DeclareMathSymbol{\ssfGamma}{0}{ssfletters}{'000}
\DeclareMathSymbol{\bsfDelta}{0}{bsfletters}{'001}
\DeclareMathSymbol{\ssfDelta}{0}{ssfletters}{'001}
\DeclareMathSymbol{\bsfTheta}{0}{bsfletters}{'002}
\DeclareMathSymbol{\ssfTheta}{0}{ssfletters}{'002}
\DeclareMathSymbol{\bsfLambda}{0}{bsfletters}{'003}
\DeclareMathSymbol{\ssfLambda}{0}{ssfletters}{'003}
\DeclareMathSymbol{\bsfXi}{0}{bsfletters}{'004}
\DeclareMathSymbol{\ssfXi}{0}{ssfletters}{'004}
\DeclareMathSymbol{\bsfPi}{0}{bsfletters}{'005}
\DeclareMathSymbol{\ssfPi}{0}{ssfletters}{'005}
\DeclareMathSymbol{\bsfSigma}{0}{bsfletters}{'006}
\DeclareMathSymbol{\ssfSigma}{0}{ssfletters}{'006}
\DeclareMathSymbol{\bsfUpsilon}{0}{bsfletters}{'007}
\DeclareMathSymbol{\ssfUpsilon}{0}{ssfletters}{'007}
\DeclareMathSymbol{\bsfPhi}{0}{bsfletters}{'010}
\DeclareMathSymbol{\ssfPhi}{0}{ssfletters}{'010}
\DeclareMathSymbol{\bsfPsi}{0}{bsfletters}{'011}
\DeclareMathSymbol{\ssfPsi}{0}{ssfletters}{'011}
\DeclareMathSymbol{\bsfOmega}{0}{bsfletters}{'012}
\DeclareMathSymbol{\ssfOmega}{0}{ssfletters}{'012}

\newcommand{\balpha}{\bm{\alpha}}

\newcommand{\bgamma}{\bm{\gamma}}

\newcommand{\btheta}{\bm{\theta}}

\newcommand{\bepsilon}{\bm{\epsilon}}

\newtheorem{theorem}{Theorem} 
\newtheorem{lemma}{Lemma}

\newtheorem{definition}{Definition} 

\newcommand{\qednew}{\nobreak \ifvmode \relax \else
      \ifdim\lastskip<1.5em \hskip-\lastskip
      \hskip1.5em plus0em minus0.5em \fi \nobreak
      \vrule height0.75em width0.5em depth0.25em\fi}

\makeatletter
\def\blfootnote{\xdef\@thefnmark{}\@footnotetext}
\makeatother

\renewcommand{\paragraph}[1]{\vspace{1.25mm}\noindent\textbf{#1}}
\newcommand{\tablestyle}[2]{\setlength{\tabcolsep}
{#1}\renewcommand{\arraystretch}{#2}\centering\small}
\newlength\savewidth

\def\eg{\emph{e.g.}}

\ificcvfinal\fi

\begin{document}

\title{\vspace{-10mm}DiffFit: Unlocking Transferability of Large Diffusion Models \\ via  Simple Parameter-Efficient Fine-Tuning \vspace{-5mm}}

\author{
    Enze Xie, 
    Lewei Yao, 
    Han Shi, 
    Zhili Liu,  
    Daquan Zhou, 
    Zhaoqiang Liu, 
    Jiawei Li, 
    Zhenguo Li 
    \\[0.2cm]
    Huawei Noah's Ark Lab
}

\twocolumn[{%
\maketitle
\vspace{-9mm}
\begin{figure}[H]
\hsize=\textwidth
\centering
\includegraphics[width=1.0\textwidth]{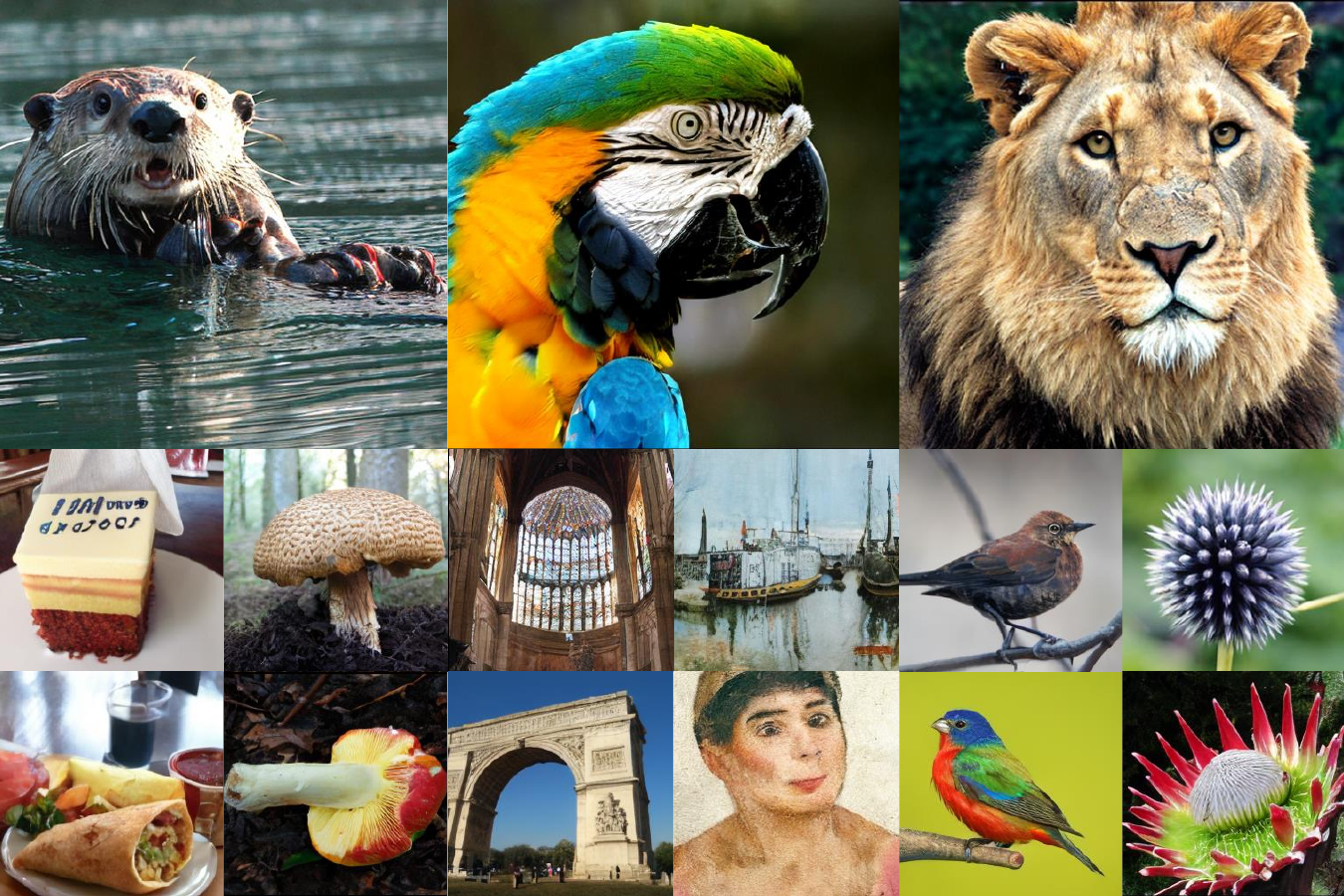}
\vspace{-7mm}
\caption{
Selected samples to show parameter-efficient fine-tuned DiT-XL/2 model using DiffFit. DiffFit only needs to fine-tune 0.12\% parameters. 
\textit{Top row}: 512$\times$512 image generation on ImageNet with \textbf{3.02 FID}. 
\textit{Bottom rows}: 256$\times$256 image generation on several downstream datasets across diverse domains: 
\textcolor{cyan}{Food},
\textcolor{orange}{Fungi},
\textcolor{violet}{Scene},
\textcolor{olive}{Art},
\textcolor{teal}{Bird},
\textcolor{purple}{Flower}.
}
\vspace{-2mm}
\label{fig:vis1}
\end{figure}
}]

\blfootnote{Correspondence to \{xie.enze, li.zhenguo\}@huawei.com}

\begin{abstract}
\vspace{-3mm}
Diffusion models have proven to be highly effective in generating high-quality images. However, adapting large pre-trained diffusion models to new domains remains an open challenge, which is critical for real-world applications.  This paper proposes DiffFit, a parameter-efficient strategy to fine-tune large pre-trained diffusion models that enable fast adaptation to new domains. DiffFit is embarrassingly simple that only fine-tunes the bias term and newly-added scaling factors in specific layers, yet resulting in significant training speed-up and reduced model storage costs. Compared with full fine-tuning, DiffFit achieves  2$\times$ training speed-up and only needs to store approximately 0.12\% of the total model parameters.  Intuitive theoretical analysis has been provided to justify the efficacy of scaling factors on fast adaptation. On 8 downstream datasets, DiffFit achieves superior or competitive performances compared to the full fine-tuning while being more efficient.  Remarkably, we show that DiffFit can adapt a pre-trained low-resolution generative model to a high-resolution one by adding minimal cost. Among diffusion-based methods, DiffFit sets a new state-of-the-art FID of 3.02 on ImageNet 512$\times$512 benchmark by fine-tuning only 25 epochs from a public pre-trained ImageNet 256$\times$256 checkpoint while being 30$\times$ more training efficient than the closest competitor.
\end{abstract}

\section{Introduction}
Denoising diffusion probabilistic models~(DDPMs)~\cite{ho2020denoising,song2021scorebased,song2019generative}
 have recently emerged as a formidable technique for generative modeling and have  demonstrated impressive results in image synthesis~\cite{dalle2,dhariwal2021diffusion,sd},
video generation~\cite{ho2022video,ho2022imagen, zhou2022magicvideo} and 3D editing~\cite{poole2022dreamfusion}.
However, the current state-of-the-art DDPMs suffer from significant computational expenses due to their large parameter sizes and numerous inference steps per image. For example, the recent \emph{DALL$\cdot$ E 2}~\cite{sd} comprises 4 separate diffusion models and requires 5.5B parameters. 
In practice, not all users are able to afford the necessary computational and storage resources. 
As such, there is a pressing need to explore methods for adapting publicly available, large, pre-trained diffusion models to suit specific tasks effectively. In light of this, a central challenge arises: \textit{Can we devise an inexpensive method to fine-tune large pre-trained diffusion models efficiently?}

Take the recent popular Diffusion Transformer (DiT) as an example, the DiT-XL/2 model, which is the largest model in the DiT family and yields state-of-the-art generative performance on the ImageNet class-conditional generation benchmark.
In detail, DiT-XL/2 comprises 640M parameters and involves computationally demanding training procedures. Our estimation indicates that the training process for DiT-XL/2 on 256$\times$256 images necessitates 950 V100 GPU days (7M iterations), whereas the training on 512$\times$512 images requires 1733 V100 GPU days (3M iterations).
The high computational cost makes training DiT from scratch unaffordable for most users.
Furthermore, extensive fine-tuning of the DiT on diverse downstream datasets requires storing multiple copies of the whole model, which results in linear storage expenditures.

In this paper, we propose DiffFit, a simple and parameter-efficient fine-tuning strategy for large diffusion models, building on the DiT as the base model. The motivation can be found in Figure~\ref{fig:diffusion_models}.
Recent work in natural language processing (BitFit~\cite{bitfit}) has demonstrated that fine-tuning only the bias term in a pre-trained model performs sufficiently well on downstream tasks.
We, therefore, seek to extend these efficient fine-tuning techniques to image generative tasks. 
We start with directly applying BitFit~\cite{bitfit} and empirically observe that simply using the BitFit technique is a good baseline for adaptation.
We then introduce learnable scaling factors $\bgamma$ to specific layers of the model, initialized to 1.0, and made dataset-specific to accommodate enhancement of feature scaling and results in better adaptation to new domains.
Interestingly, the empirical findings show that incorporating $\bgamma$ at specific locations of the model is important to reaching a better FID score. In other words, the FID score does not improve linearly with the number of $\bgamma$ included in the model.
In addition, we conducted a theoretical analysis of the mechanism underlying the proposed DiffFit for fine-tuning large diffusion models. 
We provided intuitive theoretical analysis to help understand the effect of the newly-added scaling factors in the shift of distributions. 

\begin{figure}[t!]
    \centering
    \includegraphics[width=0.45\textwidth]{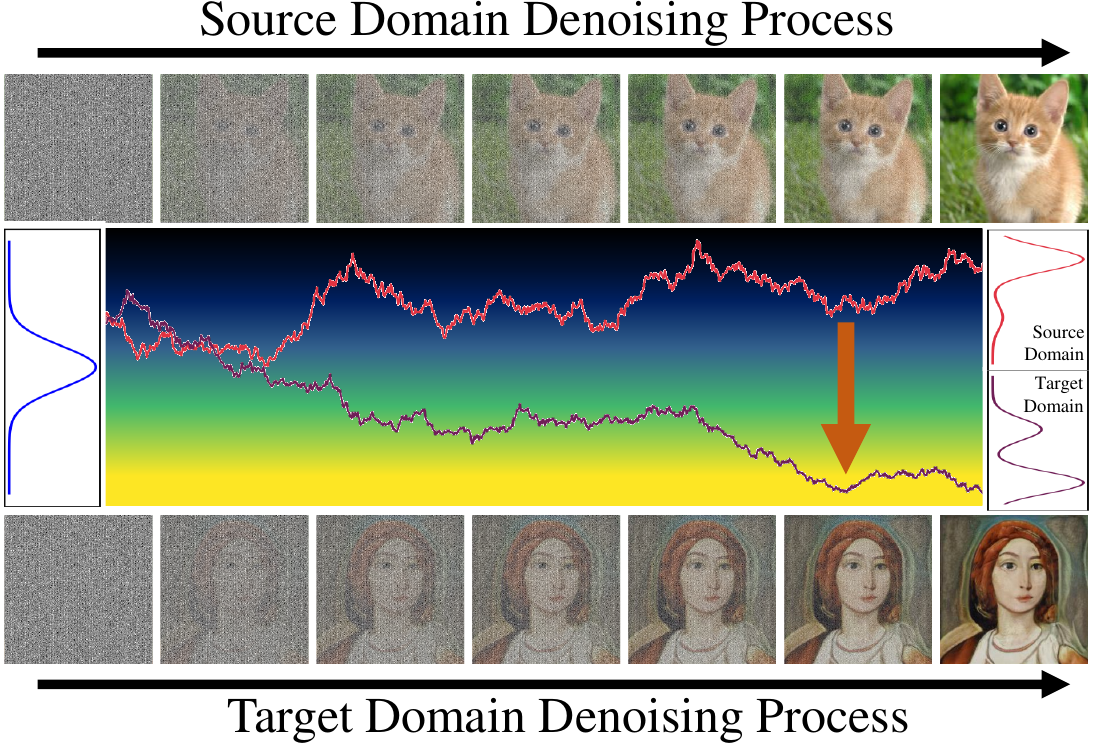}
    \caption{
    The denoising process of a diffusion model typically involves iteratively generating images from random noise.
    In DiffFit, the pre-trained large diffusion model in the source domain can be fine-tuned to adapt to a target domain with only a few specific parameter adjustments.}
    \label{fig:diffusion_models}
\end{figure}

We employed several parameter-efficient fine-tuning techniques, including BitFit~\cite{bitfit}, AdaptFormer~\cite{adaptformer}, LoRA~\cite{lora}, and VPT~\cite{vpt}, and evaluated their performance on 8 downstream datasets. Our results demonstrate that DiffFit outperforms these methods regarding Frechet Inception Distance (FID) \cite{parmar2021cleanfid} trade-off and the number of trainable parameters.
Furthermore, we surprisingly discovered that by treating high-resolution images as a special domain from low-resolution ones, our DiffFit approach could be seamlessly applied to fine-tune a low-resolution diffusion model, enabling it to adapt to high-resolution image generation at a minimal cost.
For example, starting from a pre-trained ImageNet 256$\times$256 checkpoint, by fine-tuning DIT for only 25 epochs~($\approx$0.1M iterations), DiffFit surpassed the previous state-of-the-art diffusion models on the ImageNet 512$\times$512 setting. Even though DiffFit has only about 0.9 million trainable parameters, it outperforms the original DiT-XL/2-512 model (which has 640M trainable parameters and 3M iterations) in terms of FID~(3.02 \textit{vs.} 3.04), while reducing 30$\times$ training time.
In conclusion, DiffFit aims to establish a simple and strong baseline for parameter-efficient fine-tuning in image generation and shed light on the efficient fine-tuning of larger diffusion models.

Our contributions can be summarized as follows:

\begin{enumerate}[itemsep=5pt,topsep=0pt,parsep=0pt]
    \item We propose a simple parameter-efficient fine-tuning approach for diffusion image generation named DiffFit. It achieves superior results compared to full fine-tuning while leveraging only 0.12\% trainable parameters. Quantitative evaluations across 8 downstream datasets demonstrate that DiffFit outperforms existing well-designed fine-tuning strategies (as shown in Figure~\ref{fig:bubbles} and Table~\ref{table: fgvc}).
    \item We conduct an intuitive theoretical analysis and design detailed ablation studies to provide a deeper understanding of why this simple parameter-efficient fine-tuning strategy can fast adapt to new distributions. 
    \item We show that by treating high-resolution image generation as a downstream task of the low-resolution pre-trained generative model, DiffFit can be seamlessly extended to achieve superior generation results with FID 3.02 on ImageNet and reducing training time by 30 times, thereby demonstrating its scalability.

\end{enumerate}

\section{Related Works}

\subsection{Transformers in Vision}
Transformer architecture was first introduced in language model~\cite{vaswani2017attention} and became dominant because of its scalability, powerful performance and emerging ability~\cite{radford2018improving,radford2019language,brown2020language}.
Then, Vision Transformer~(ViT)~\cite{dosovitskiy2020image} and its variants achieved colossal success and gradually replaced ConvNets in various visual recognition tasks, \eg ~image classification~\cite{deit,deepvit,t2tvit,tnt}, object detection~\cite{swin,pvt,pvtv2,detr}, semantic segmentation~\cite{setr,segformer,segmenter} and so on~\cite{transtrack,li2022panoptic,zhao2021point,liu2022video,he2022masked,li2022bevformer}. 
Transformers are also widely adopted in GAN-based generative models~\cite{taming,transgan} and the conditional part of text-to-image diffusion models~\cite{sd,imagen,dalle2,ediffi}.
Recently, DiT~\cite{dit} proposed a plain Transformer architecture for the denoising portion of diffusion networks and verified its scaling properties.
Our paper adopts DiT as a strong baseline and studies parameter-efficient fine-tuning.

\subsection{Diffusion Models}
Diffusion models \cite{ho2020denoising} (aka. score-based models \cite{songscore}) have shown great success in generative tasks, including density estimation \cite{kingma2021variational}, image synthesis \cite{dhariwal2021diffusion}, text-to-image generation \cite{sd,ediffi,sahariaphotorealistic} and so on. Different from previous generative models like GAN \cite{creswell2018generative}, VAE \cite{kingma2019introduction} and Flow \cite{rezende2015variational}, diffusion models \cite{ho2020denoising} transform a data distribution to a Gaussian distribution by progressively adding noise, and then, reversing the process via denoising to retrieve the original distribution. The progressive step-by-step transformation between the two distributions makes the training process of diffusion models more stable compared to other models. However, the multiple time-step generations makes the diffusion process time-consuming and expensive.

\begin{figure}[t]
\includegraphics[width=\linewidth]{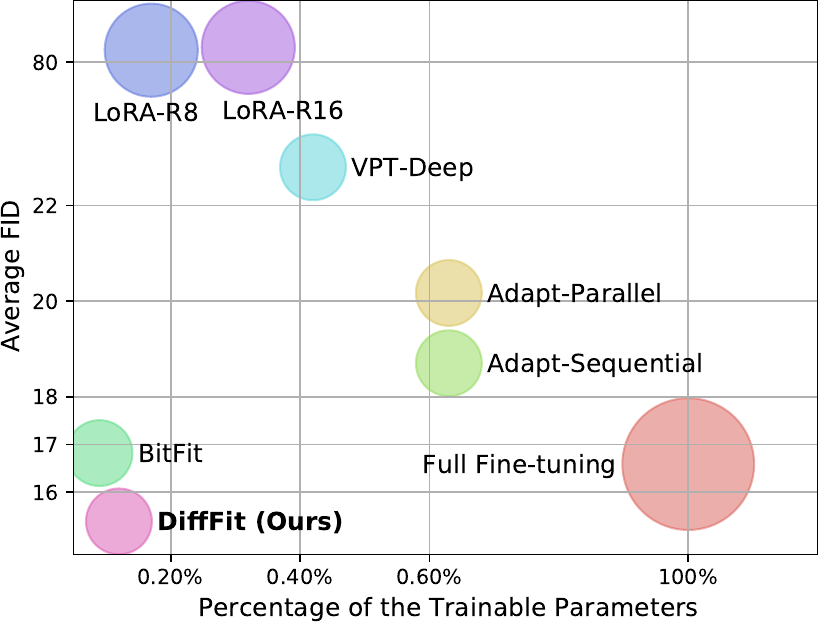}
\caption{
Average FID score of fine-tuned DiT across 8 downstream datasets.
The radius of each bubble reflects the training time~(smaller is better). 
We observe that DiffFit performs remarkably well in terms of achieving the best FID while requiring fewer computations and parameters.
}
\vspace{-4mm}
\label{fig:bubbles}
\end{figure}

\subsection{Parameter-efficient Fine-tuning}
Witnessing the success of Transformers in language and vision, many large models based on Transformer architecture have been developed and pre-trained on massive upstream data. On the one hand, the industry continues to increase the model parameters to billion, even trillion scales~\cite{brown2020language,fedus2021switch,vit22b} to probe up the upper bound of large models. On the other hand, fine-tuning and storage of large models are expensive. 
There are three typical ways for parameter-efficient fine-tuning as follows:

1. Adaptor~\cite{houlsby2019parameter,lora,adaptformer}. Adaptor is a small module inserted between Transformer layers, consisting of a down-projection, a nonlinear activation function, and an up-projection. Specifically, LoRA \cite{lora} adds two low-rank matrices to the query and value results between the self-attention sub-layer. AdaptFormer \cite{adaptformer}, however, places the trainable low-rank matrices after the feed-forward sub-layer.

2. Prompt Tuning~\cite{li2021prefix,lester2021power,logan2021cutting,vpt,zhou2022learning}.
Usually, prefix tuning \cite{li2021prefix} appends some tunable tokens before the input tokens in the self-attention module at each layer. In contrast, prompt-tuning \cite{lester2021power} only appends the tunable tokens in the first layer for simplification. VPT \cite{vpt} focuses on the computer vision field and proposes deep and shallow prompt tuning variants. 

3. Partial Parameter Tuning~\cite{bitfit,xu2021raise,lian2022scaling}. Compared with the above parameter-efficient methods, partial parameter tuning does not insert any other components and only fine-tunes the partial parameters of the original model. For example, BitFit \cite{bitfit} tunes the bias of each linear projection and Child-Tuning \cite{xu2021raise} evaluates the importance of parameters and tunes only the important ones.   

\section{Methodology}

\subsection{Preliminaries}
\paragraph{Diffusion Models.}
Denoising diffusion probabilistic models (DDPMs) \cite{ho2020denoising} define generative models by adding Gaussian noise gradually to data and then reversing back. Given a real data sample $\bx_0 \sim q_{data}(\bx)$, the forward process is controlled by a Markov chain as $q(\bx_t|\bx_{t-1})=\mathcal{N}(\bx_t;\sqrt{1-\beta_t}\bx_{t-1}, \beta_t \bI)$, where $\beta_t$ is a variance schedule between $0$ and $1$. By using the reparameterization trick, we have $\bx_t=\sqrt{\bar{\alpha}_t}\bx_0+\sqrt{1-\bar\alpha_t}\bepsilon$, where $\bepsilon\sim\mathcal{N}(\bm{0}, \bI)$, $\alpha_t = 1 - \beta_t$ and $\bar\alpha_t=\prod_{i=1}^t\alpha_i$. For larger time step $t$, we have smaller $\bar{\alpha}_t$, and the sample gets noisier.

As for the reverse process, DDPM learns a denoise neural network $p_{\btheta}(\bx_{t-1}|\bx_t)=\mathcal{N}(\bx_{t-1};\bm{\mu}_{\btheta}(\bx_t, t), \sigma_t^2 \bI)$. The corresponding objective function is the following variational lower bound of the negative log-likelihood:
\begin{equation}
    \footnotesize
    \mathcal{L}(\btheta)=\sum_t\mathcal{D}_{\mathrm{KL}}\left(q(\bx_{t-1}|\bx_t,\bx_0) \big| p_{\btheta}(\bx_{t-1}|\bx_t)\right)-p_{\btheta}(\bx_0|\bx_1),
\end{equation}
where $\mathcal{D}_{\mathrm{KL}}(p\big|q)$ represents the KL divergence measuring the distance between two distributions $p$ and $q$. Furthermore, the objective function can be reduced to $\mathcal{L}_{vlb}=\mathbb{E}_{\bx_0, \bepsilon, t}\big[\frac{\beta_t^2}{2\alpha_t(1-\bar\alpha_t)\sigma_t^2}\|\bepsilon-\bepsilon_{\btheta}\|^2\big]$ and a simple variant loss function $L_{\mathrm{simple}}=\mathbb{E}_{\bx_0, \bepsilon, t}\big[\|\bepsilon-\bepsilon_{\btheta}\|^2\big]$. Following iDDPM \cite{nichol2021improved}, we use a hybrid loss function as $\mathcal{L}_{hybrid}=\mathcal{L}_{simple} + \lambda\mathcal{L}_{vlb}$, where $\lambda$ is set to be $0.001$ in our experiments.

\paragraph{Diffusion Transformers~(DiT).}
Transformer \cite{vaswani2017attention} architecture has proved to be powerful in image recognition, and its design can be migrated to diffusion models for image generation. DiT~\cite{dit} is a recent representative method that designs a diffusion model with Transformers. 
DiT follows the design of latent diffusion models~(LDMs)~\cite{sd}, which have two parts given a training sample $\bx$: (1) An autoencoder consisting of an encoder $E$ and a decoder $D$, where the latent code $\bz = E(\bx)$ and the reconstructed data $\hat{\bx} = D(\bz)$;
(2) A latent diffusion transformer with patchify, sequential DiT blocks, and depatchify operation. In each block $B_i$, we have $\bz_i=B_i(\bx, t, c)$, where $t$ and $c$ are time embedding and class embedding. Each block $B_i$ contains a self-attention and a feed-forward module. The patchification/depatchification operations are used to encode/decode latent code $\bz$ to/from a sequence of image tokens.

\subsection{Parameter-efficient Fine-tuning}

\begin{figure}[t!]
    \centering
    \includegraphics[width=0.46\textwidth]{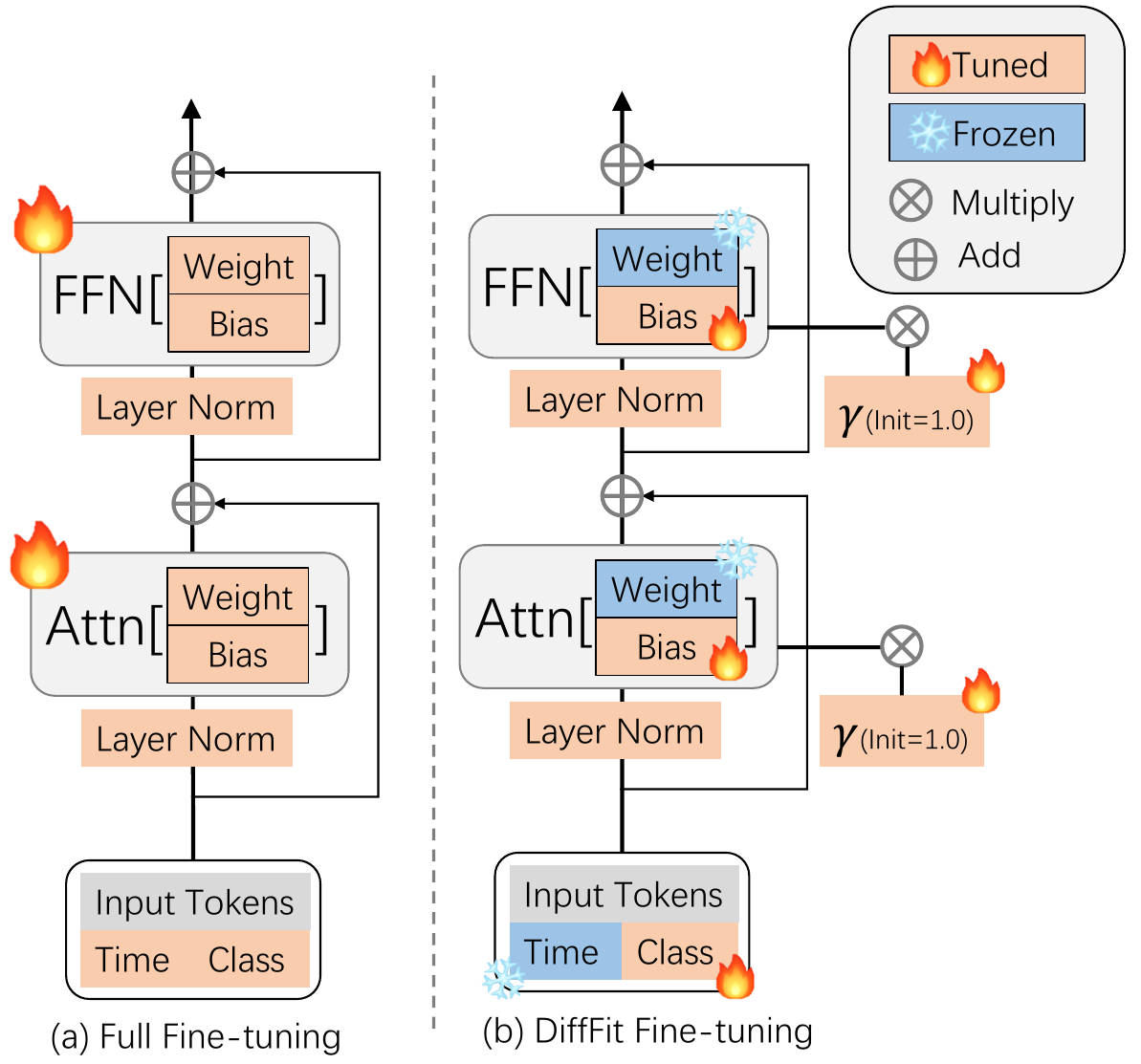}
    \caption{
    Architecture comparison between full fine-tuning and the proposed DiffFit. DiffFit is simple and effective, where most of the weights are frozen and only bias-term, scale factor $\bgamma$, LN, and class embedding are trained. 
    }
    \vspace{-4mm}
    \label{fig:method}
\end{figure}

\paragraph{DiffFit Design.}
This section illustrates the integration of DiT with DiffFit.
Note that DiffFit may be generalized to other diffusion models \eg\ Stable Diffusion. 
Our approach, illustrated in Figure~\ref{fig:method}, involves freezing the majority of parameters in the latent diffusion model and training only the bias term, normalization, and class condition module. 
We moreover insert learnable scale factors $\bgamma$ into several blocks of the diffusion model, wherein the $\bgamma$ is initialized to 1.0 and multiplied on corresponding layers of each block. 
Each block typically contains multiple components such as multi-head self-attention, feed-forward networks, and layer normalization, and the block can be stacked $N$ times.
Please refer to Algorithm~\ref{alg:gamma} for additional detailed information.

\begin{figure*}[t]
\vspace{-5mm}
\centering
\hspace{-1em}
    \begin{minipage}{0.45\linewidth}{
        \centering
        \input{tabs/alg-1.tex}}
    \end{minipage}
\hspace{2em}
    \begin{minipage}{0.47\linewidth}{
        \centering
        \input{tabs/alg-2.tex}}
    \end{minipage}
\end{figure*}

\begin{table*}[ht]
	\centering
	\vspace{-0.3cm}
	\setlength{\tabcolsep}{4pt}
	\scalebox{0.90}{\begin{tabular}{l|c|c|c|c|c|c|c|c|c|c|c}
			\toprule
			\diagbox{Method}{Dataset}&\makecell[c]{~Food~} & ~SUN~ & \makecell[c]{DF-20M}  & \makecell[c]{Caltech}  & \makecell[c]{CUB-Bird} & \makecell[c]{ArtBench} & \makecell[c]{~Oxford~ \\ Flowers} & \makecell[c]{~Standard~ \\ Cars} &  \makecell[c]{Average \\ FID} & \makecell[c]{~Params. \\ (M)} & \makecell[c]{Train \\ Time}  \\
            \midrule
			Full Fine-tuning                     & 10.46 & \textbf{7.96} & \textbf{17.26} & 35.25 & \underline{5.68} & 25.31  & 21.05 & \textbf{9.79} & \underline{16.59} & 673.8~(100\%)  & 1$\times$\\
			Adapt-Parallel~\cite{adaptformer}    & 13.67 & 11.47 & 22.38 & 35.76 & 7.73 & 38.43  & 21.24 & 10.73 & 20.17 & 4.28~(0.63\%) & 0.47$\times$\\
            Adapt-Sequential  & 11.93 & 10.68 & 19.01 & \underline{34.17} & 7.00 & 35.04  & 21.36 & 10.45 & 18.70 & 4.28~(0.63\%) & 0.43$\times$ \\
			BitFit~\cite{bitfit}                 & \underline{9.17} & 9.11 & 17.78 & 34.21 & 8.81 & \underline{24.53}   & \underline{20.31} & 10.64 & 16.82 & \textbf{0.61~(0.09\%)} & 0.45$\times$\\
			VPT-Deep~\cite{vpt}                  & 18.47 & 14.54 & 32.89 & 42.78 & 17.29 & 40.74 & 25.59 & 22.12 & 26.80 & 2.81~(0.42\%) & 0.50$\times$ \\
            LoRA-R8~\cite{lora}                  & 33.75 & 32.53 & 120.25 & 86.05 & 56.03 & 80.99 & 164.13 & 76.24 & 81.25 & 1.15~(0.17\%) & 0.63$\times$\\
            LoRA-R16                 & 34.34 & 32.15 & 121.51 & 86.51 & 58.25 & 80.72 & 161.68 & 75.35 & 81.31 & 2.18~(0.32\%) & 0.68$\times$ \\
            \midrule
            \rowcolor{gray!15}
			DiffFit \textbf{(ours)}              & \textbf{6.96} & \underline{8.55} & \underline{17.35} & \textbf{33.84} & \textbf{5.48} & \textbf{20.87}   & \textbf{20.18} & \underline{9.90} & \textbf{15.39} & \underline{0.83~(0.12\%)} & 0.49$\times$
            \tabularnewline
			\bottomrule
		\end{tabular}
 	}
        \vspace{-0.2cm}
	\caption{FID performance comparisons on 8 downstream datasets with DiT-XL-2 pre-trained on ImageNet 256$\times$256.}
	\label{table: fgvc}
	\vspace{-0.2cm}
\end{table*}

\paragraph{Fine-tuning.}
During fine-tuning, diffusion model parameters are initially frozen, after which only specific parameters related to bias, class embedding, normalization, and scale factor are selectively unfrozen. Our approach, outlined in Algorithm~\ref{alg:train}, enables fast fine-tuning while minimizing disruption to pre-trained weights. DiT-XL/2 requires updating only 0.12\% of its parameters, leading to training times approximately 2$\times$ faster than full fine-tuning. 
Our approach avoids catastrophic forgetting while reinforcing the pre-trained model's knowledge and enabling adaptation to specific tasks.

\paragraph{Inference and Storage.}
After fine-tuning on $K$ datasets, we only need to store one copy of the original model's full parameters and $K \times$ dataset-specific trainable parameters, typically less than 1M for the latter. Combining these weights for the diffusion model enables adaptation to multiple domains for class-conditional image generation.

\subsection{Analysis}

In this subsection, we provide intuitive theoretical justifications for the efficacy of scaling factors and reveal the principle behind their effectiveness. We note that these theoretical justifications are intended as a simple proof of concept rather than seeking to be comprehensive, as our contributions are primarily experimental.

Specifically, recent theoretical works for diffusion models, such as \cite{de2021diffusion,deconvergence,leeconvergence,chen2022sampling,chen2023score}, have shown that under suitable conditions on the data distribution and the assumption of approximately correct score matching, diffusion models can generate samples that approximately follow the data distribution, starting from a standard Gaussian distribution. Given a mapping $\bm{f}$ and a distribution $P$, we denote $\bm{f} \# P$ as a pushforward measure, i.e., for any measurable $\Omega$, we have $(\bm{f} \# P)(\Omega) = P(\bm{f}^{-1}(\Omega))$. Note that our base model DiT is pre-trained on ImageNet with a resolution of $256 \times 256$ and is fine-tuned on downstream datasets with the same resolution but much fewer data points and classes. Motivated by this, if assuming that the data in the ImageNet $256 \times 256$ dataset follows a distribution $Q_0$, then we can assume that the data in the downstream dataset follows a distribution $P_0 = \bm{f}_{\bgamma^*} \# Q_0$, where $\bm{f}_{\bgamma^*}$ is a linear mapping dependent on some ground-truth scaling factors $\bgamma^*$. 

With these assumptions in place, we can formulate an intuitive theorem that provides insight into the effectiveness of scaling factors. A formal version of this theorem is included in the supplementary material.

\begin{theorem}[informal]
\label{thm:difffit_informal}
    Suppose that for a dataset generated from data distribution $Q_0$, we can train a neural network such that the diffusion model generates samples that approximately follow $Q_0$. Further assuming that the data distribution $P_0$ for a relatively small dataset can be written as $P_0 = \bm{f}_{\bgamma^*} \# Q_0$ with $\bm{f}_{\bgamma^*}$ being a linear mapping dependent on ground-truth scaling factors $\bgamma^*$. Then, if only re-training the neural network with the goal of optimizing scaling factors (and all other parameters remain unchanged), under suitable conditions, a simple gradient descent algorithm seeks an estimate $\hat{\bgamma}$ that is close to $\bgamma^*$ with high probability. Furthermore, with the fine-tuned neural network corresponding to $\hat{\bgamma}$, the denoising process produces samples following a distribution that is close to $P_0$. 
\end{theorem}
In summary, Theorem~\ref{thm:difffit_informal} essentially states that when distributions $Q_0$ and $P_0$ satisfy the condition that $P_0 = \bm{f}_{\bgamma^*} \# Q_0$ with $\bm{f}_{\bgamma^*}$ being dependent on ground-truth scaling factors $\bgamma^*$, diffusion models can transfer from distribution $Q_0$ to $P_0$ in the denoising process with fine-tuning the scaling factors in the training process.
\section{Experiments}
\begin{figure*}[t]
\hsize=\textwidth
\vspace{-5mm}
\centering
\includegraphics[width=1.0\textwidth]{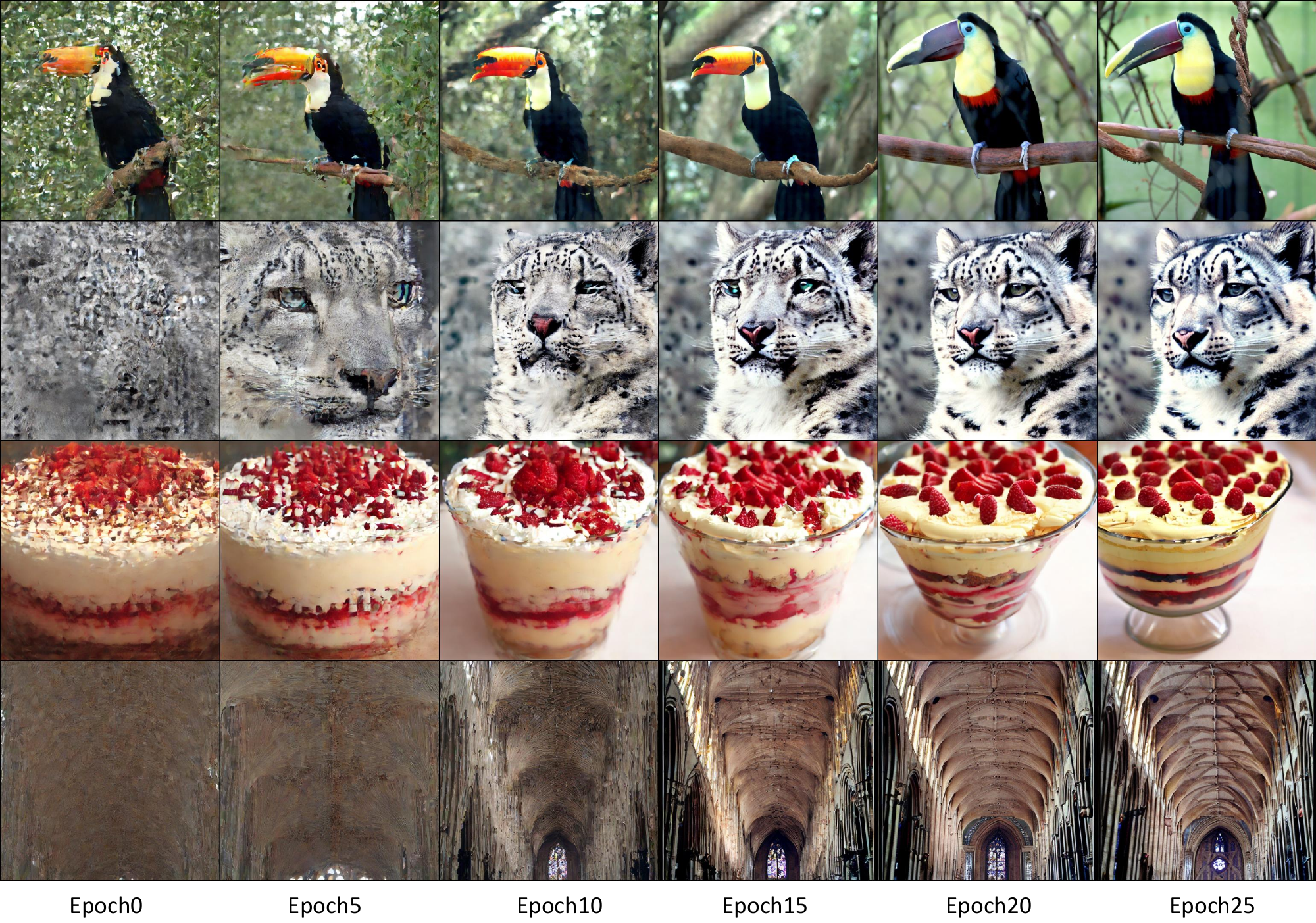}
\caption{
Fine-tune DiT-XL/2-512 from the checkpoint of DiT-XL/2-256 using DiffFit with the positional encoding trick.
}
\label{fig:learning process}
\end{figure*}

\subsection{Implementation Details}
Our base model is the DiT\footnote{https://github.com/facebookresearch/DiT}, which is pre-trained on ImageNet 256$\times$256 with 7 million iterations, achieving an FID score of 2.27\footnote{https://dl.fbaipublicfiles.com/DiT/models/DiT-XL-2-256x256.pt}.
However, since the original DiT repository does not provide training code, we re-implemented this and achieved reasonable results. 
Following DiT, we set the constant lr=1e-4 for full fine-tuning and set the classifier-free guidance to 1.5 for evaluation and 4.0 for visualization.
In addition, we re-implemented several parameter-efficient fine-tuning methods such as Adaptor, BitFit, Visual Prompt Tuning (VPT), and LoRA.
We found VPT to be sensitive to depth and token number while training was exceptionally unstable. As such, we sought for a more stable configuration with depth=5, token=1, and kept the final layers unfrozen for all tasks. 
We enlarge lr $\times$10 for parameter-efficient fine-tuning settings to obtain better results following previous works~\cite{lora,adaptformer}.

\subsection{Transfer to Downstream Datasets} \label{sec:downstream}
\paragraph{Setup.} For fine-tuning downstream small datasets with 256$\times$256 resolutions, we use 8 V100 GPUs with a total batch size of 256 and train 24K iterations.
We choose 8 commonly used fine-grained datasets: Food101, SUN397, DF-20M mini, Caltech101, CUB-200-2011, ArtBench-10, Oxford Flowers and Stanford Cars. 
We report FID using 50 sampling steps for all the tasks.
Most of these datasets are selected from CLIP downstream tasks except ArtBench-10 since  it has distinct distribution from ImageNet, which enables a more comprehensive evaluation of the out-of-distribtion generalization performance of our DiffFit.

\paragraph{Results.} 
We list the performance of different parameter-efficient fine-tuning methods in Table~\ref{table: fgvc}. 
As can be seen, by tuning only $0.12\%$ parameters, our DiffFit achieves the lowest FID on average over $8$ downstream tasks. 
While full fine-tuning is a strong baseline and has slightly better results on 3/8 datasets, it is necessary to fine-tune 100\% parameters.
Among all baselines, the performance of LoRA is surprisingly poor. As discussed in \cite{he2022parameter}, LoRA performs worse on image classification tasks than other parameter-efficient fine-tuning methods. As the image generation task is conceptually more challenging than the image classification task, it is reasonable that the performance gap between LoRA and other approaches becomes larger here.

\begin{table}[t]
\centering
\begin{tabular}{l|cc}
\toprule
Method & FID $\downarrow$ & \makecell{Training Cost \\ (GPU days)} $\downarrow$ \\
\midrule
BigGAN-Deep~\cite{brock2018large} & 8.43 &  256-512  \\
StyleGAN-XL~\cite{sauer2022stylegan} & 2.41 & 400 \\ 
\midrule
ADM-G, AMD-U~\cite{dhariwal2021diffusion} & 3.85 &  1914 \\
DiT-XL/2~\cite{dit} & 3.04 &  1733  \\
\midrule
\rowcolor{gray!10}
DiffFit \textbf{(ours)} & \textbf{3.02} & \textbf{51} \small\color{gray}{(+950$^\dag$)}  \\

\bottomrule
\end{tabular}
\vspace{-.5em}
\caption{\textbf{Class-conditional image generation on ImageNet 512$\times$512.} The training cost of the original DiT model and our method is measured on V100 GPU devices, and other methods are quoted from original papers. $\dag$: 950 GPU days indicates the pre-training time of DiT-XL/2 model on ImageNet 256$\times$256 with 7M steps. } 
\label{tab:IN512_performance} 
\end{table}

\begin{table}[t]
\centering
\begin{tabular}{l|cc}
\toprule
Method &  \makecell{Pre-trained \\ Checkpoint} & FID $\downarrow$  \\
\midrule
Full Fine-tune & IN256(7M)$\rightarrow$Food256 &  23.08 \\
\midrule
DiffFit & IN512~(3M) &  19.25  \\
\rowcolor{gray!10}
DiffFit & IN256~(7M)$\rightarrow$Food256 &  19.50  \\
\rowcolor{gray!20}
~~+PosEnc Trick & IN256~(7M)$\rightarrow$Food256 &  \textbf{19.10}  \\
\bottomrule
\end{tabular}
\vspace{-.5em}
\caption{\textbf{Class-conditional image generation on Food-101 512$\times$512.} 
We use two pre-trained models from: (1) ImageNet 512$\times$512, and (2)  ImageNet 256$\times$256 and firstly fine-tuned on Food-101 256$\times$256.} 
\label{tab:Food512_performance} 
\end{table}

\begin{table*}[t]
\centering
\hspace{-1em}

\subfloat[
    \textbf{Scale: Deep$\rightarrow$Shallow}. 
    \label{tab:d2s}
]{
    \centering
    \begin{minipage}{0.20\linewidth}{
        \begin{center}
            \tablestyle{4pt}{1.05}
            \begin{tabular}{l|cc}
            \toprule
            Blocks & \#params~(M) & FID $\downarrow$ \\
            \midrule
            28$\rightarrow$25 & 0.747 &  10.04  \\
            28$\rightarrow$22 & 0.754 &  10.03 \\ 
            28$\rightarrow$18 & 0.763 &  10.33 \\
            28$\rightarrow$14 & 0.770 &  10.51  \\
            28$\rightarrow$11 & 0.779 &  9.92  \\
            28$\rightarrow$8 & 0.786 &   9.28  \\
            28$\rightarrow$4 & 0.796 &   8.87  \\
            \rowcolor{gray!15}
            28$\rightarrow$1 & 0.803 &   \textbf{8.19} \\
            \bottomrule
            \end{tabular}
        \end{center}}
    \end{minipage}
}
\hspace{2em}
\subfloat[
    \textbf{Scale: Shallow$\rightarrow$Deep}.
    \label{tab:s2d}
]{
    \begin{minipage}{0.20\linewidth}{
        \begin{center}
            \tablestyle{4pt}{1.05}
            \begin{tabular}{l|cc}
            \toprule
            Blocks & \#params~(M) & FID $\downarrow$ \\
            \midrule
            1$\rightarrow$3 & 0.745 &  8.29  \\
            1$\rightarrow$7 & 0.754 &  7.99 \\ 
            1$\rightarrow$11 & 0.763 &  7.72 \\
            \rowcolor{gray!15}
            1$\rightarrow$14 & 0.770 &  \textbf{7.61}  \\
            1$\rightarrow$18 & 0.779 &  7.63  \\
            1$\rightarrow$21 & 0.786 &  7.67  \\
            1$\rightarrow$25 & 0.796 &  7.85  \\
            1$\rightarrow$28 & 0.803 &  8.19  \\
            \bottomrule
            \end{tabular}
        \end{center}}
    \end{minipage}
}
\hspace{2em}
\subfloat[
    \textbf{Scale Location}.
    \label{tab:scale_location}
]{
    \begin{minipage}{0.20\linewidth}{
        \begin{center}
            \tablestyle{2pt}{1.05} 
            \begin{tabular}{c|l|cc}
            \toprule
            \#ID& Scale Factor & FID $\downarrow$ \\
            \midrule
            1& NA~(BitFit)  &  9.17  \\
            2& +Blocks &   8.19  \\
            3& \cellcolor{red!10}~~+PatchEmb  &  \cellcolor{red!10}9.05  \\
            4& \cellcolor{red!10}~~+TimeEmb &  \cellcolor{red!10}8.46  \\
            5& \cellcolor{green!10}~~+QKV-Linear  &  \cellcolor{green!10}7.37  \\
            6& \cellcolor{green!10}~~+Final Layer  &  \cellcolor{green!10}7.49  \\
            7 & +ID 1, 2, 5, 6  &  7.17  \\
            8 & \cellcolor{gray!15}~~+(1$\rightarrow$14 Layers)  &  \cellcolor{gray!15}\textbf{6.96}  \\
            \bottomrule
            \end{tabular}
        \end{center}}
    \end{minipage}
}
\hspace{2em}
\subfloat[
    \textbf{Learning Rate}.
    \label{tab:lr}
]{
    \begin{minipage}{0.15\linewidth}{
        \begin{center}
            \tablestyle{4pt}{1.05}
            \begin{tabular}{cc}
            \toprule
            LR Ratio & FID $\downarrow$ \\
            \midrule
            0.1$\times$ &  25.85  \\
            0.2$\times$ &  21.42  \\
            0.5$\times$ &  17.16  \\
            1$\times$ &  15.68  \\
            2$\times$ &  13.84  \\
            5$\times$ &  10.97  \\
            \rowcolor{gray!15}
            10$\times$ &  \textbf{8.19}  \\
            20$\times$ &  8.30 \\
            \bottomrule
            \end{tabular}
        \end{center}}
    \end{minipage}
}
\vspace{-0.2cm}
\caption{\textbf{Ablation experiments on Food101 dataset with DiT-XL/2.}
\colorbox{red!15}{Red} means adding scale factor leads to negative results and  \colorbox{green!15}{green} means positive. 
The best setting are marked in \colorbox{gray!15}{gray}.}
\label{tab:ablations}
\end{table*}

\subsection{From Low Resolution to High Resolution}
\paragraph{Setup.} 
Considering the generating  images with different resolutions  as a special type of distribution shift, 
our proposed method can effortlessly adapt a pre-trained low-resolution diffusion model to generate high-resolution images.
To demonstrate the effectiveness of DiffFit, we load a pre-trained ImageNet $256\times256$ DiT-XL/2 checkpoint and fine-tune the model on the ImageNet $512\times512$.
We employ a positional encoding trick to speed up fine-tuning.
We fine-tune DiT-XL/2 on ImageNet $512\times512$ using 32 V100 GPUs with 1024 batch size and 30K iterations.
We report FID using 250 sampling steps.
Note that we do not need to fine-tune the label embedding here since the label does not change.

\paragraph{Positional Encoding Trick.} 
DiT~\cite{dit} adopts a static sinusoidal 2D positional encoding scheme. To better utilize the positional information encoded in the pre-trained model, we develop a sinusoidal interpolation that aligns the positional encoding of 512$\times$512 resolution with that of 256$\times$256 resolution. 
This is implemented by replacing each pixel coordinate $(i,j)$ in the positional encoding formula with its half value $(i/2,j/2)$, which is simple and  have no extra costs.

\begin{figure*}[t]
\centering

\scalebox{1}{
\includegraphics[width=\linewidth]{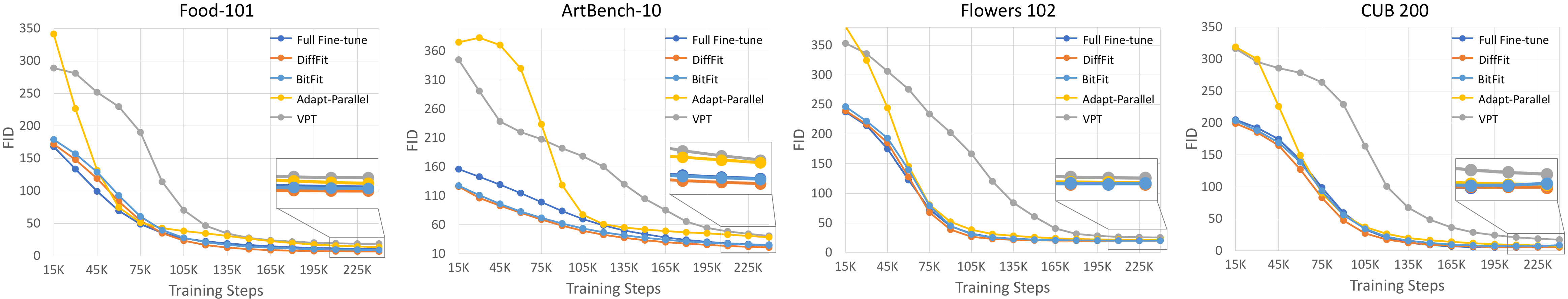}
}
\caption{
\textbf{FID of five methods every 15K iterations on four downstream datasets.}
Our observations indicate that DiffFit  can rapidly adapt to the target domain while maintaining a robust FID score.
}
\vspace{-4mm}
\label{fig:curev1}
\end{figure*}

\paragraph{Results.} 
As demonstrated in Table~\ref{tab:IN512_performance}, 
DiffFit achieves 3.02 FID on the ImageNet $512\times512$ benchmark, surpassing ADM's 3.84 and the official DiT's 3.04. 
DiffFit sets a new state-of-the-art among diffusion-based methods.
Importantly, our method is significantly more efficient than the previous methods, as the fine-tuning process only requires an overhead of 51 GPU days. Even when accounting for the 950 GPU days of pre-training, our method remains superior to the 1500+ GPU days required by DiT and ADM. 
We observe faster training convergence using the positional encoding trick, as shown in Figure~\ref{fig:learning process}. For more visualizations please see Figure~\ref{fig:vis1} and appendix.

In Table~\ref{tab:Food512_performance}, we conducted fine-tuning experiments on the Food101 dataset with a resolution of 512$\times$512 using the DiffFit method. Our results reveal that fine-tuning a pre-trained Food101 256$\times$256 checkpoint with DiffFit yields a FID of 4 improvements over full fine-tuning. 
Interestingly, we found that utilizing a pre-trained ImageNet 512$\times$512 checkpoint leads to a FID performance similar to that achieved by the pre-trained Food101 with 256$\times$256 resolution. Moreover, we observed a slight improvement in FID performance by incorporating the proposed positional encoding trick into the fine-tuning process.

\subsection{Fine-tuning Convergence Analysis}
To facilitate the analysis of converging speed, we present the FID scores for several methods every 15,000 iterations in the Food-101, ArtBench-10, Flowers-102 and CUB-200 datasets, as shown in Figure~\ref{fig:curev1}. 
Our observations demonstrate that full fine-tuning, BitFit, and our proposed DiffFit exhibit similar convergence rates, which surpass the initial performance of AdaptFormer and VPT. 
While AdaptFormer initially presents inferior performance, it shows a rapid improvement in the middle of training. In contrast, VPT exhibits the slowest rate of convergence. 
Compared to full fine-tuning, DiffFit freezes most of the parameters and thus maximally preserves the information learned during the pre-training and thus achieves a better fine-tuned performance. 
Compared to BifFit, DiffFit adjusts the feature using scale factor, resulting in faster convergence and better results.

The above observation and analysis verify that our proposed DiffFit demonstrates fast adaptation abilities to target domains and an excellent image generation capability.

\subsection{Ablation studies}
\paragraph{Setup.}
We conduct ablation studies on Food101 dataset, which has 101 classes. Each class has 750/250 images in the train/test set. Other settings are the same as Section~\ref{sec:downstream}.

\paragraph{Scale factor $\bgamma$ in different layers.} 
We investigate the effect of the scaling factor $\bgamma$ via 2 designs: gradually adding $\bgamma$  from deep to shallow~(Table~\ref{tab:d2s}) and from shallow to deep~(Table~\ref{tab:s2d}).
The results demonstrate that adding $\bgamma$ before the 14th layer of DiT-XL/2 gradually increases the performance from 8.29 to 7.61 FID while adding more $\bgamma$ in the deeper layers hurts the performance. 
We hypothesize that deeper layers are responsible for learning high-level features that capture abstract data patterns, which contribute to synthesizing the final output and are often complex and non-linear. Adding $\bgamma$ in deeper layers poses a risk of disrupting the learned correlations between the high-level features and data, which might negatively impact the model.

\paragraph{Scale factor $\bgamma$ in different modules.} 
We study the impact of scaling factor $\bgamma$ in various DiT modules, as illustrated in Table~\ref{tab:scale_location}. 
Based on BitFit~\cite{bitfit}'s FID score of 9.17, incorporating the scaling factor $\bgamma$ in transformer blocks and QKV-linear layers of self-attention can significantly enhance the performance, resulting in a FID score of 7.37 (row 5). However, introducing scale factors in other modules, such as patch embedding (row 3) and time embedding (row 4), does not bring noticeable improvements. By integrating the effective designs, i.e., adding $\bgamma$ in blocks, QKV-linear and final layers, we improve FID score to 7.17 (row 7). We further improve the FID score of 6.96 by placing $\bgamma$ in the first 14 blocks (Table \ref{tab:s2d}). This optimal setting is adopted as the final setting for our approach.

\paragraph{Learning rate.} 
Adjusting the learning rate is a crucial step in fine-tuning. 
Parameter-efficient fine-tuning typically requires a larger learning rate than the pre-training~\cite{lora,adaptformer} since pre-training has already initialized most of the model's  parameters to a certain extent and a larger learning rate can help  quickly adapt the remaining parameters to the new tasks. 
We perform a learning rate search on our method, as shown in Table~\ref{tab:lr}.
We observe that using a learning rate 10$\times$ greater than pre-training yields the best result. Larger learning rates than 10$\times$ resulted in decreased performance and even unstable training.

\section{Conclusions and Limitations}
In this paper, we propose DiffFit, a straightforward yet effective  fine-tuning approach that can quickly adapt a large pre-trained diffusion model to various downstream domains, including different datasets or varying resolutions. By only fine-tuning bias terms and scaling factors, DiffFit provides a cost-effective solution to reduce storage requirements and speed up fine-tuning without compromising performance. 
One limitation is that our experiments mainly focus on class-conditioned image generation. It is still unclear whether this strategy could perform equally well in more complex tasks such as text-to-image generation or video/3D generation. We leave these areas for future research. 

\paragraph{Acknowledgement.} We would like to express our gratitude to Junsong Chen, Chongjian Ge, and Jincheng Yu for their assistance with experiments on LLaMA and DreamBooth.

\newpage
\onecolumn

{\centering
    {\huge \bf Supplementary Material}
    \\
    {\Large \bf DiffFit: Unlocking Transferability of Large Diffusion Models \\ via  Simple Parameter-efficient Fine-Tuning \\
     } 
}

\section{More Applications}
\subsection{Combine DiffFit with ControlNet}
\paragraph{Setup.} 
To verify that DiffFit can be applied to large-scale text-to-image generation models, we conducted experiments on the recent popular ControlNet~\cite{controlnet} model. We used the official code of ControlNet\footnote{\url{https://github.com/lllyasviel/ControlNet}},
and chose the semantic segmentation mask as the additional condition. 
We used the COCO-Stuff~\cite{cocostuff} as our dataset and the text prompt was generated by the BLIP~\cite{blip}. 
The model was trained on the training set and evaluated on the validation set using 8 V100 GPUs with a batch size of 2 per GPU and 20 epochs. Unless otherwise stated, we followed the official implementation of ControlNet for all hyper-parameters including the learning rate and the resolution.

\paragraph{Combination.} 
ControlNet~\cite{controlnet} enhances the original Stable Diffusion (SD) by incorporating a conditioning network comprising two parts: a set of zero convolutional layers initialized to zero and a trainable copy of 13 layers of SD initialized from a pre-trained model.

In our DiffFit setting, we freeze the entire trainable copy part and introduce a scale factor of $\gamma$ in each layer. Then, we unfreeze the $\gamma$ and all the bias terms. 
Note that in ControlNet, the zero convolutional layers are required to be trainable. 
Table~\ref{tab:ControlNet} shows that DiffFit has 11.2M trainable parameters while the zero convolutional layers contribute 10.8M parameters. Reducing the trainable parameters in zero convolutional layers is plausible, but it is beyond the scope of this paper.

\paragraph{Results.}
Table~\ref{tab:ControlNet} displays the results obtained from the original ControlNet and the results from combining DiffFit with ControlNet. Our results show that the addition of DiffFit leads to comparable FID and CLIP scores while significantly reducing the number of trainable parameters. 
In addition, we note that ControlNet primarily focuses on fine-grained controllable generation while FID and CLIP score may not accurately reflect the overall performance of a model.
Figures~\ref{fig:controlnet-1} and~\ref{fig:controlnet-2} compare the results from ControlNet and DiffFit. When using the same mask and text prompt as conditions, fine-tuning with DiffFit produces more visually appealing results compared to fine-tuning with the original ControlNet.

Our experimental results suggest that DiffFit has great potential to improve the training of other types of advanced generative models.

\subsection{Combine DiffFit with DreamBooth}
\paragraph{Setup.} Our DiffFit can also be easily combined with the DreamBooth~\cite{dreambooth}. DreamBooth is a personalized text-to-image diffusion model that fine-tunes a pre-trained text-to-image model using a few reference images of a specific subject, allowing it to synthesize fully-novel photorealistic images of the subject contextualized in different scenes. 
Our implementation on DreamBooth utilizes the codebase from Diffusers\footnote{\url{https://github.com/huggingface/diffusers/tree/main/examples/dreambooth}}, with Stable Diffusion as the base text-to-image model.

\paragraph{Implementation detail.} 
Given 3-5 input images, the original DreamBooth fine-tunes the subject embedding and the entire text-to-image (T2I) model. In contrast, our DiffFit simplifies the fine-tuning process by incorporating scale factor $\gamma$ into all attention layers of the model and requiring only the $\gamma$, bias term, LN, and subject embedding to be fine-tuned. 

\paragraph{Results.}
We compare the original full-finetuning approach of DreamBooth with the fine-tuning using LoRA and DiffFit, as shown in Figures~\ref{fig:dreambooth-1} and \ref{fig:dreambooth-2}.
First, we observe that all three methods produce similar generation quality. 
Second, we find that the original full-finetuning approach is storage-intensive, requiring the storage of 859M parameters for one subject-driven fine-tuning. In contrast, LoRA and DiffFit are both parameter-efficient (requiring the storage of less than 1\% parameters), with DiffFit further reducing trainable parameters by about 26\% compared to LoRA.

\subsection{Combine DiffFit with LLaMA and Alpaca}

\paragraph{Setup.} 
LLaMA~\cite{llama} is an open and efficient large language model~(LLM) from Meta, ranging from 7B to 65B parameters. 
Alpaca is a model fully fine-tuned from the LLaMA 7B model on 52K instruction-following demonstrations and behaves qualitatively similarly to OpenAI’s ChatGPT model \texttt{text-davinci-003}. 

\paragraph{Combination.} 
We use the original code from Alpaca\footnote{\url{https://github.com/tatsu-lab/stanford_alpaca}} as the baseline and compare it with Alpaca-LoRA\footnote{\url{https://github.com/tloen/alpaca-lora}}. Since LLaMA and DiT have a similar Transformer design, the fine-tuning strategy of LLaMA-DiffFit is exactly the same as that of DiT-DiffFit, except that the learning rate is set to 3e-4 to align with Alpaca-LoRA.

\paragraph{Results.}
In Table \ref{tab:instruction_comparisons_short}, we showcase the instruction tuning results on llama using Alpaca, LoRA, and DiffFit. 
The trainable parameters of Alpaca, LoRA, and DiffFit are \textbf{\textit{7B, 6.7M, and 0.7M}}.
The model fine-tuned with DiffFit exhibits strong performance in natural language question-answering tasks (e.g., common sense questioning, programming, etc.). This validates the efficacy of DiffFit for instruction tuning in NLP models, highlighting DiffFit as a flexible and general fine-tuning approach. 

\begin{table}[h]
\centering
\begin{tabular}{l|cccc}
\toprule
Method & FID $\downarrow$ & CLIP Score $\uparrow$ & Total Params~(M) & Trainable Params~(M)\\
\midrule
ControlNet & 20.1 & 0.3067 & 1343 & 361 \\ 
ControlNet + DiffFit & 19.5 &  0.3064  & 1343 & 11.2\\
\bottomrule
\end{tabular}
\vspace{-.5em}
\caption{Results of original ControlNet and combined DiffFit on COCO-Stuff dataset.} 
\label{tab:ControlNet} 
\end{table}

\begin{figure*}[h]
    \centering
    \includegraphics[width=1.0\textwidth]{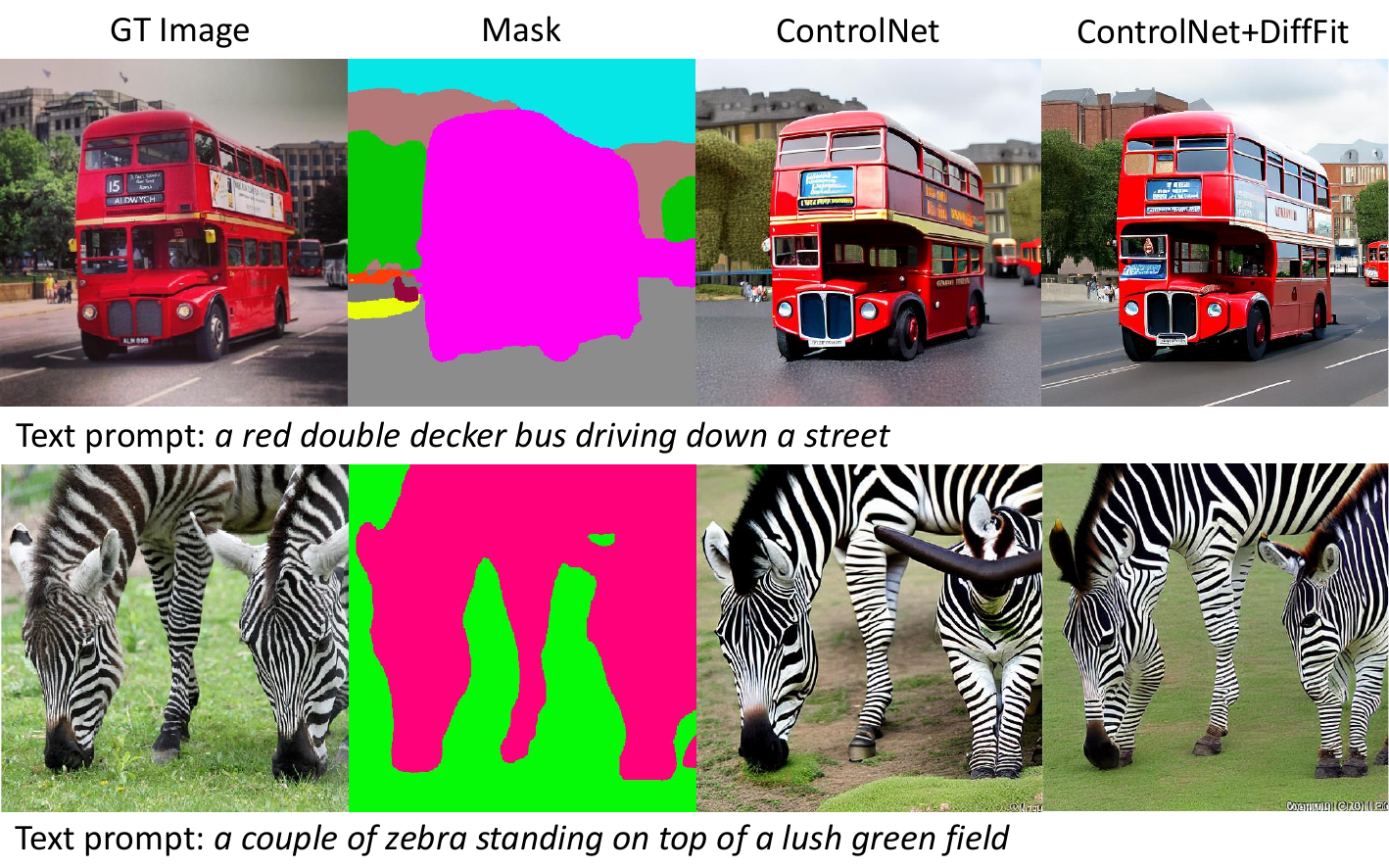}
    \caption{
    Visualization of original ControlNet and combined DiffFit on COCO-Stuff dataset. 
    }
    \label{fig:controlnet-1}
\end{figure*}

\begin{figure*}[h]
    \centering
    \includegraphics[width=1.0\textwidth]{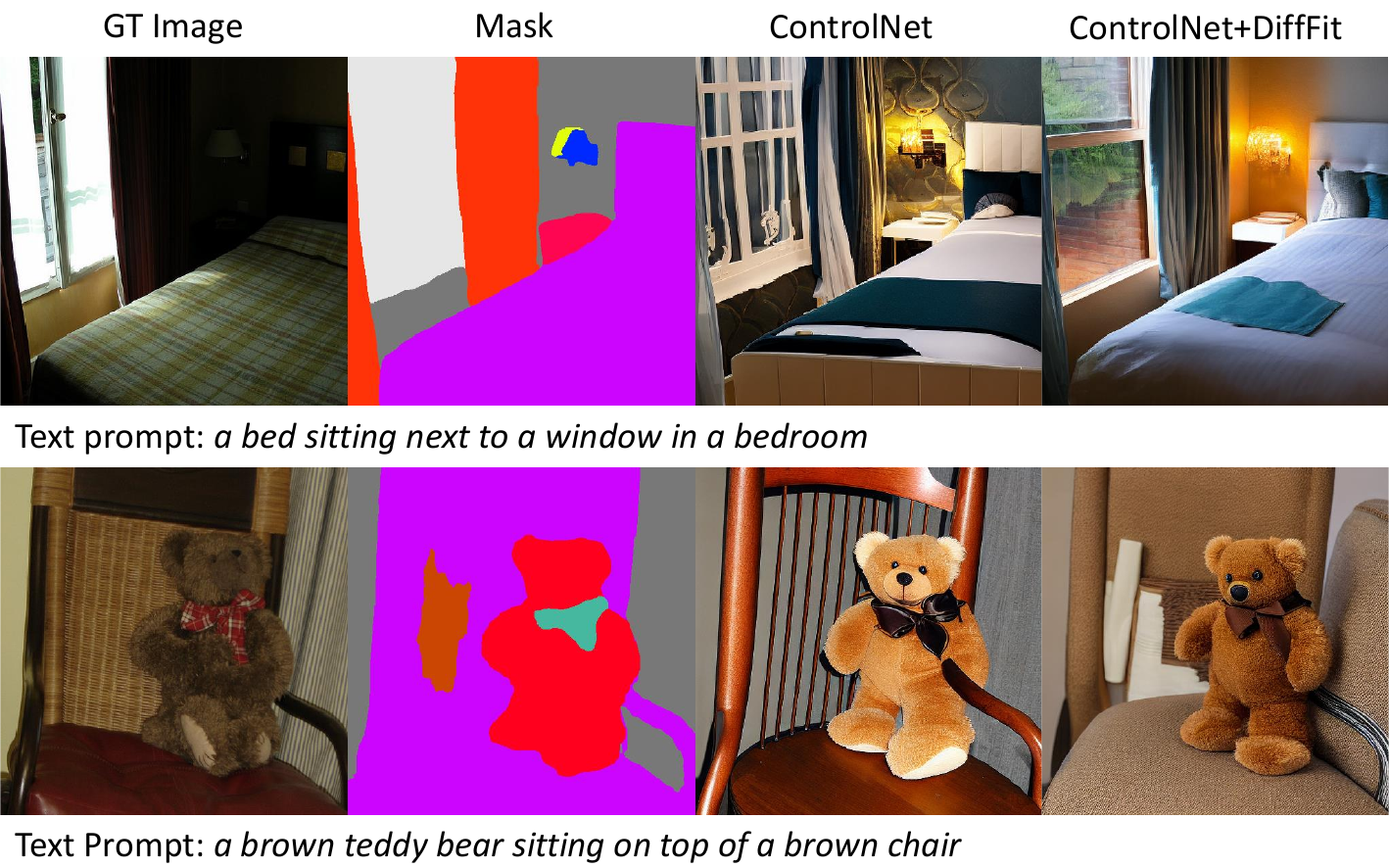}
    \caption{
    Visualization of original ControlNet and combined DiffFit on COCO-Stuff dataset. 
    }
    \label{fig:controlnet-2}
\end{figure*}

\begin{figure*}[h]
    \centering
    \includegraphics[width=1.0\textwidth]{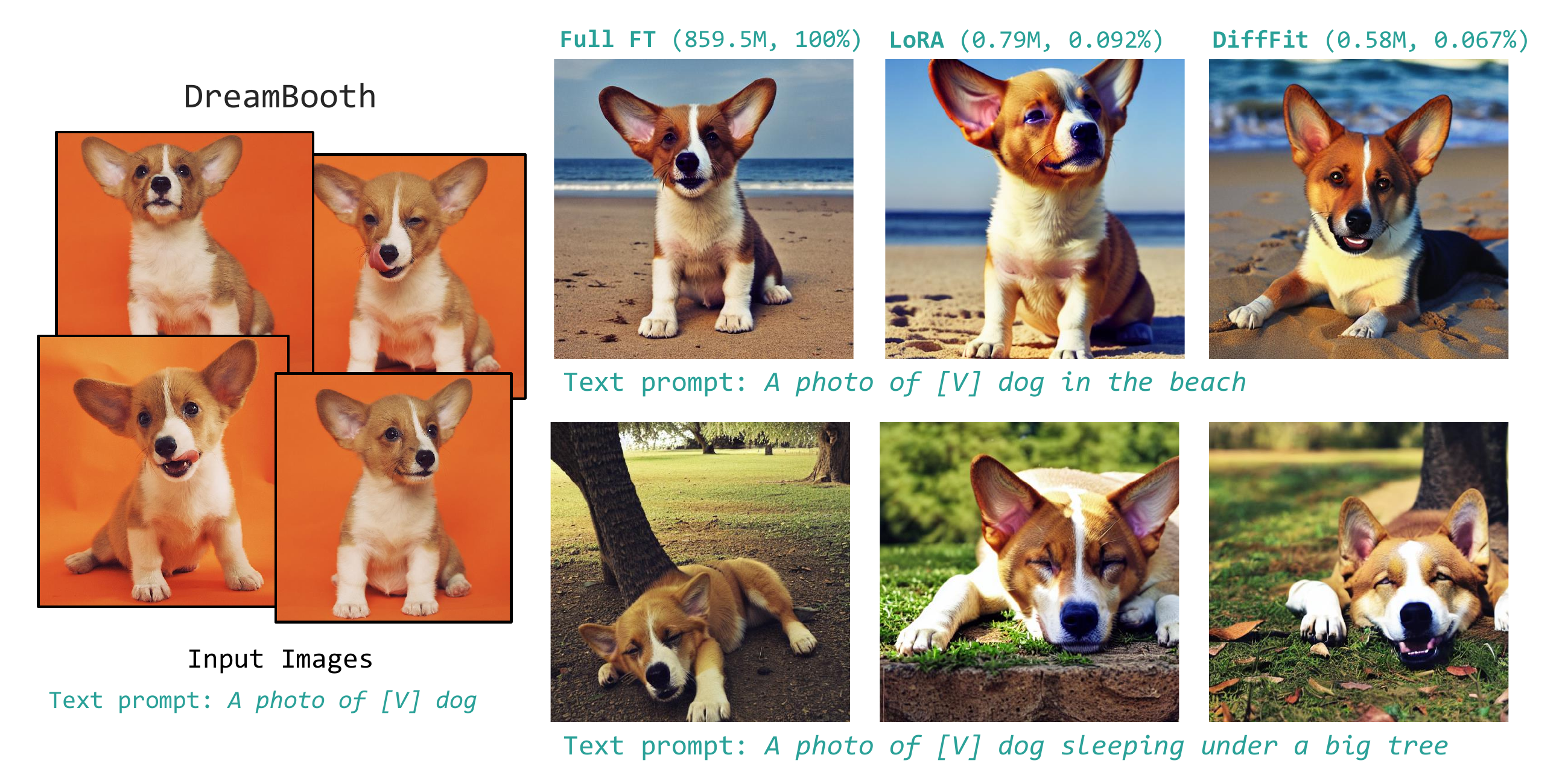}
    \caption{
    Visualization of original DreamBooth, with LoRA and DiffFit. 
    }
    \label{fig:dreambooth-1}
\end{figure*}

\begin{figure*}[h]
    \centering
    \includegraphics[width=1.0\textwidth]{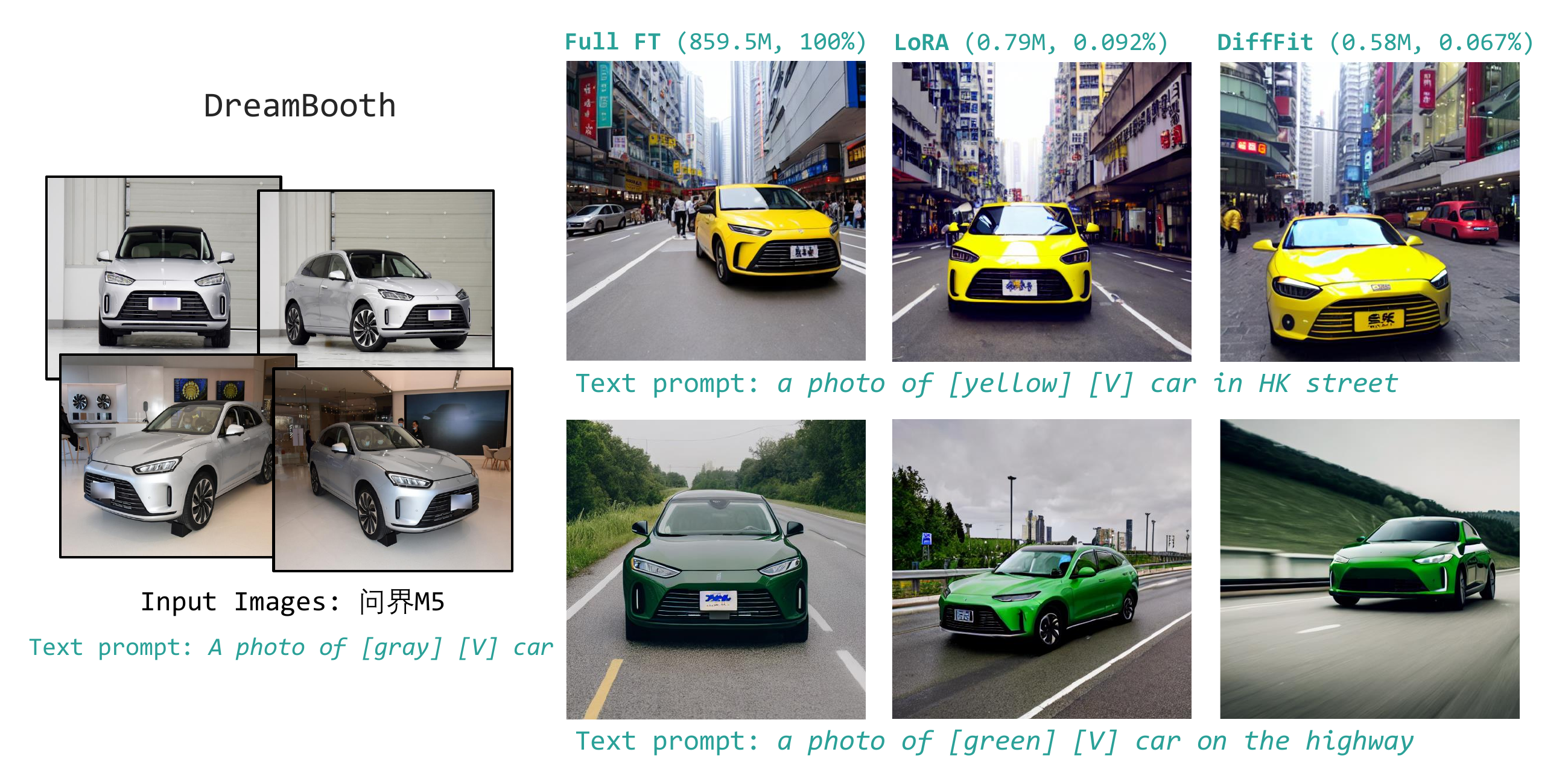}
    \caption{
    Visualization of original DreamBooth, with LoRA and DiffFit. 
    }
    \label{fig:dreambooth-2}
\end{figure*}

\begin{table}[h]
\centering
\footnotesize
\begin{tabular}{p{1.0\columnwidth}}
\definecolor{codeblue}{rgb}{0.25,0.5,0.5}
\definecolor{codekw}{rgb}{0.85, 0.18, 0.50}
\lstset{
  backgroundcolor=\color{lightgray},
  basicstyle=\fontsize{7.5pt}{7.5pt}\ttfamily\selectfont,
  columns=fullflexible,
  breaklines=true,
  captionpos=b,
  commentstyle=\fontsize{7.5pt}{7.5pt}\color{codeblue},
  keywordstyle=\fontsize{7.5pt}{7.5pt}\color{codekw},
  escapechar={|}, 
}
\begin{framed}
\centering
 \begin{minipage}{\textwidth}
{\bf Instruction:} Tell me about alpacas. \\
{\bf Alpaca:} Alpacas are small, fluffy animals related to camels and llamas. They are native to Peru and Bolivia, and were first domesticated around 5,000 years ago. They are kept mainly for their fine, soft fleece, which is used to make knitwear and other garments. Alpacas are herd animals and live in small family groups, led by an older male. They are highly social animals and can form strong bonds with their owners. \\
{\bf Alpaca-LoRA:}  Alpacas are members of the camelid family and are native to the Andes Mountains of South America. They are known for their soft, luxurious fleece, which is used to make clothing, blankets, and other items. Alpacas are herbivores and graze on grasses and other plants. They are social animals and live in herds of up to 20 individuals. \\
{\bf LLaMA-DiffFit:} Alpacas are members of the camelid family and are native to the Andes Mountains in South America. They are domesticated for their soft, luxurious fleece, which is used to make clothing, blankets, and other textiles. Alpacas are herbivores and graze on grasses and other plants. They can live up to 20 years in captivity and are known for their gentle nature. 
\end{minipage}
\end{framed}
\vspace{-5mm}
\begin{framed}
\centering
 \begin{minipage}{\textwidth}
{\bf Instruction:} List all Canadian provinces in alphabetical order. \\
{\bf Alpaca:} Alberta, British Columbia, Manitoba, New Brunswick, Newfoundland and Labrador, Nova Scotia, Ontario, Prince Edward Island, Quebec, Saskatchewan, Northwest Territories, Nunavut, Yukon. \\
{\bf Alpaca-LoRA:} Alberta, British Columbia, Manitoba, New Brunswick, Newfoundland and Labrador, Nova Scotia, Ontario, Prince Edward Island, Quebec, Saskatchewan. \\
{\bf LLaMA-DiffFit:} Alberta, British Columbia, Manitoba, New Brunswick, Newfoundland and Labrador, Northwest Territories, Nova Scotia, Nunavut, Ontario, Prince Edward Island, Quebec, Saskatchewan, Yukon.
\end{minipage}
\end{framed}
\vspace{-5mm}
\begin{framed}
\centering
 \begin{minipage}{\textwidth}
{\bf Instruction:} Tell me five words that rhyme with 'shock'. \\
{\bf Alpaca:} Five words that rhyme with shock are: rock, pop, shock, cook, and snock. \\
{\bf Alpaca-LoRA:} Flock, lock, rock, stock, and sock. \\
{\bf LLaMA-DiffFit:}  Shock, sock, rock, cook, and book.
\end{minipage}
\end{framed}
\vspace{-5mm}
\begin{framed}
\centering
 \begin{minipage}{\textwidth}
{\bf Instruction:} Write a Python program that prints the first 10 Fibonacci numbers. \\
{\bf Alpaca:} 
\begin{lstlisting}[language=python]
numbers = [0, 1]

for i in range(2, 11):
    numbers.append(numbers[i-2] + numbers[i-1])

print(numbers)
# Output: [0, 1, 1, 2, 3, 5, 8, 13, 21, 34]
\end{lstlisting}
{\bf Alpaca-LoRA:}
\begin{lstlisting}[language=python]
def fibonacci(n):
    if n == 0:
        return 0
    elif n == 1:
        return 1
    else:
        return fibonacci(n-1) + fibonacci(n-2)

print(fibonacci(10))
\end{lstlisting}
{\bf LLaMA-DiffFit:}
\begin{lstlisting}[language=python]
def fib(n):
    if n == 0:
        return 0
    elif n == 1:
        return 1
    else:
        return fib(n-1) + fib(n-2)

print(fib(10))
\end{lstlisting}
\end{minipage}
\end{framed}
\end{tabular}
\vspace{-5mm}
\caption{\textbf{Instruction-following Comparison} between Alpaca, Alpaca-LoRA and LLaMA-DiffFit.}
\label{tab:instruction_comparisons_short}
\end{table}

\section{More Implementation Details}
\subsection{Adding the Scale Factor $\bgamma$ in Self-Attention}
This section demonstrates the incorporation of scale factor $\gamma$ into the self-attention modules. Traditionally, a self-attention layer comprises two linear operations: (1) QKV-Linear operation and (2) Project-Linear operation. Consistent with the approach discussed above in the paper, the learnable scale factor is initialized to 1.0 and subsequently fine-tuned. Algorithm~\ref{alg:gamma_attn} presents the Pytorch-style pseudo code.

\begin{algorithm}[h]
\small
\caption{\small Adding the trainable scale factor $\gamma$ in the attention block.
}
\label{alg:gamma_attn}
\definecolor{codeblue}{rgb}{0.25,0.5,0.5}
\definecolor{codekw}{rgb}{0.85, 0.18, 0.50}
\lstset{
  backgroundcolor=\color{white},
  basicstyle=\fontsize{7.5pt}{7.5pt}\ttfamily\selectfont,
  columns=fullflexible,
  breaklines=true,
  captionpos=b,
  commentstyle=\fontsize{7.5pt}{7.5pt}\color{codeblue},
  keywordstyle=\fontsize{7.5pt}{7.5pt}\color{codekw},
  escapechar={|}, 
}
\begin{lstlisting}[language=python]
import torch
import torch.nn as nn

class Attention(nn.Module):
  def __init__(self, dim, num_heads=8, qkv_bias=False, proj_bias=True, attn_drop=0., proj_drop=0., eta=None):
    super().__init__()
    self.num_heads = num_heads
    head_dim = dim // num_heads
    self.scale = head_dim ** -0.5

    self.qkv = nn.Linear(dim, dim * 3, bias=qkv_bias)
    self.attn_drop = nn.Dropout(attn_drop)
    self.proj = nn.Linear(dim, dim, bias=proj_bias)
    self.proj_drop = nn.Dropout(proj_drop)

    # Initilize gamma to 1.0
    self.gamma1 = nn.Parameter(torch.ones(dim * 3))
    self.gamma2 = nn.Parameter(torch.ones(dim))
    

  def forward(self, x):
    B, N, C = x.shape
    # Apply gamma
    qkv = (self.gamma1 * self.qkv(x)).reshape(B, N, 3, self.num_heads, C//self.num_heads).permute(2,0,3,1,4)
    q, k, v = qkv.unbind(0) 

    attn = (q @ k.transpose(-2, -1)) * self.scale
    attn = attn.softmax(dim=-1)
    attn = self.attn_drop(attn)

    x = (attn @ v).transpose(1, 2).reshape(B, N, C)
    # Apply gamma
    x = self.gamma2 * self.proj(x)
    x = self.proj_drop(x)
    return x
\end{lstlisting}
\end{algorithm}

\subsection{Dataset Description}
This section describes 8 downstream datasets/tasks that we have utilized in our experiments to fine-tune the DiT with our DiffFit method.

\paragraph{Food101~\cite{bossard2014food}.}
This dataset contains 101 food categories, totaling 101,000 images. Each category includes 750 training images and 250 manually reviewed test images. The training images were kept intentionally uncleaned, preserving some degree of noise, primarily vivid colors and occasionally incorrect labels. All images have been adjusted to a maximum side length of 512 pixels.

\paragraph{SUN 397~\cite{xiao2010sun}.}
The SUN benchmark database comprises 108,753 images labeled into 397 distinct categories. The quantities of images vary among the categories, however, each category is represented by a minimum of 100 images. These images are commonly used in scene understanding applications.

\paragraph{DF20M~\cite{picek2022danish}.}
DF20 is a new fine-grained dataset and benchmark featuring highly accurate class labels based on the taxonomy of observations submitted to the Danish Fungal Atlas. The dataset has a well-defined class hierarchy and a rich observational metadata. It is characterized by a highly imbalanced long-tailed class distribution and a negligible error rate. Importantly, DF20 has no intersection with ImageNet, ensuring unbiased comparison of models fine-tuned from ImageNet checkpoints.

\paragraph{Caltech 101~\cite{griffin2007caltech}.}
The Caltech 101 dataset comprises photos of objects within 101 distinct categories, with roughly 40 to 800 images allocated to each category. The majority of the categories have around 50 images. Each image is approximately 300$\times$200 pixels in size.

\paragraph{CUB-200-2011~\cite{Wah2011TheCB}.}
CUB-200-2011 (Caltech-UCSD Birds-200-2011) is an expansion of the CUB-200 dataset by approximately doubling the number of images per category and adding new annotations for part locations. The dataset consists of 11,788 images divided into 200 categories.

\paragraph{ArtBench-10~\cite{liao2022artbench}.}
ArtBench-10 is a class-balanced, standardized dataset comprising 60,000 high-quality images of artwork annotated with clean and precise labels. It offers several advantages over previous artwork datasets including balanced class distribution, high-quality images, and standardized data collection and pre-processing procedures. It contains 5,000 training images and 1,000 testing images per style.

\paragraph{Oxford Flowers~\cite{flowers}.}
The Oxford 102 Flowers Dataset contains high quality images of 102 commonly occurring flower categories in the United Kingdom. The number of images per category range between 40 and 258. This extensive dataset provides an excellent resource for various computer vision applications, especially those focused on flower recognition and classification.

\paragraph{Stanford Cars~\cite{cars}.}
In the Stanford Cars dataset, there are 16,185 images that display 196 distinct classes of cars. These images are divided into a training and a testing set: 8,144 images for training and 8,041 images for testing. The distribution of samples among classes is almost balanced. Each class represents a specific make, model, and year combination, e.g., the 2012 Tesla Model S or the 2012 BMW M3 coupe.

\section{More FID Curves}
We present additional FID curves for the following datasets: Stanford Cars, SUN 397, Caltech 101, and DF20M as illustrated in Figures~\ref{fig:curve_car} through~\ref{fig:curve_df20m}, respectively. The results demonstrate that our DiffFit achieves rapid convergence, delivers compelling FID scores compared to other parameter-efficient finetuning strategies, and is competitive in the full-finetuning setting.

\begin{figure*}[h]
\vspace{-5mm}
\centering
\hspace{-1em}
    \begin{minipage}{0.48\linewidth}{
        \centering
        \includegraphics[width=1.0\textwidth]{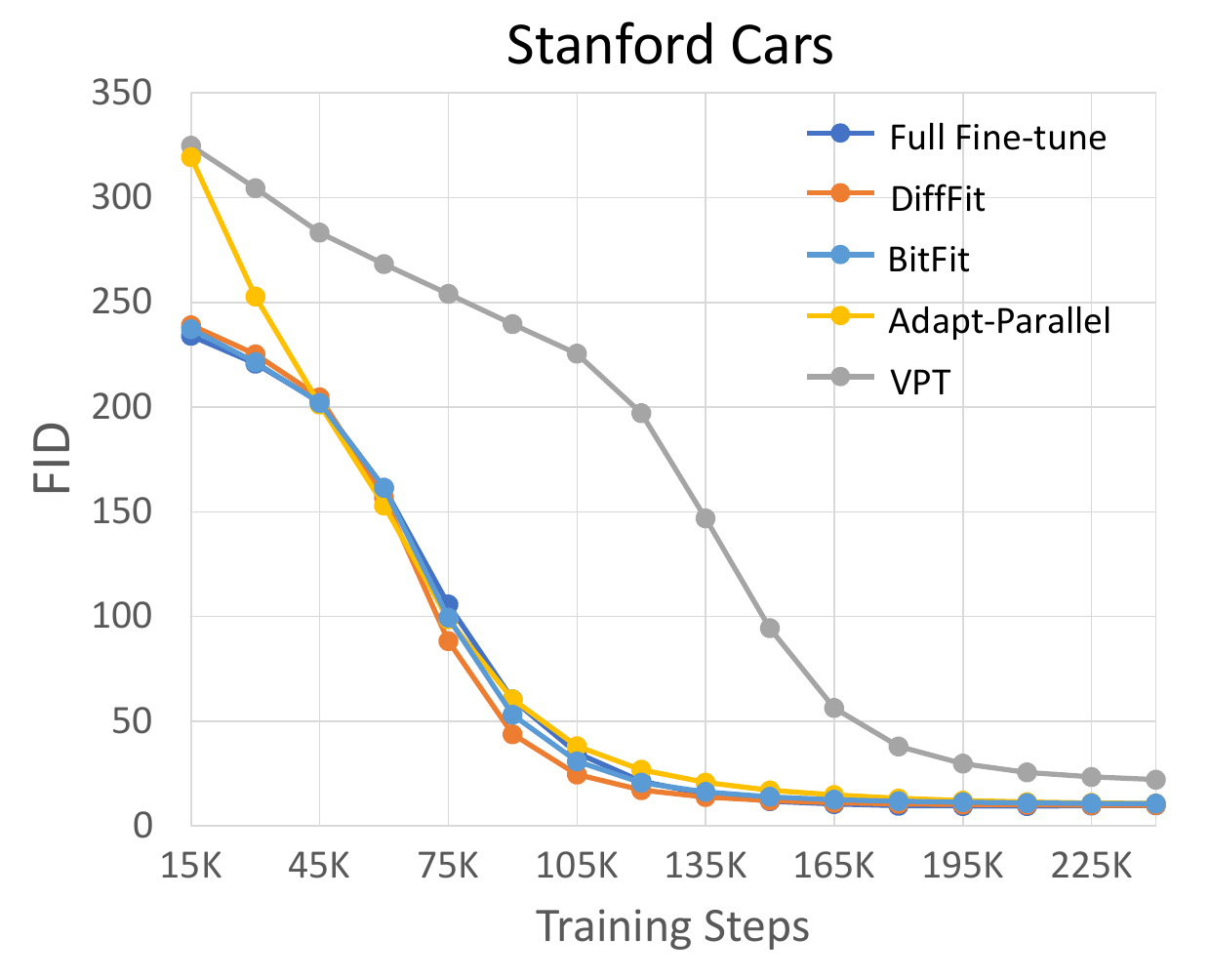}
        \caption{FID of five methods every 15K iterations on Stanford Cars dataset.}
        \label{fig:curve_car}
        }
    \end{minipage}
\hspace{2em}
    \begin{minipage}{0.48\linewidth}{
        \centering
        \includegraphics[width=1.0\textwidth]{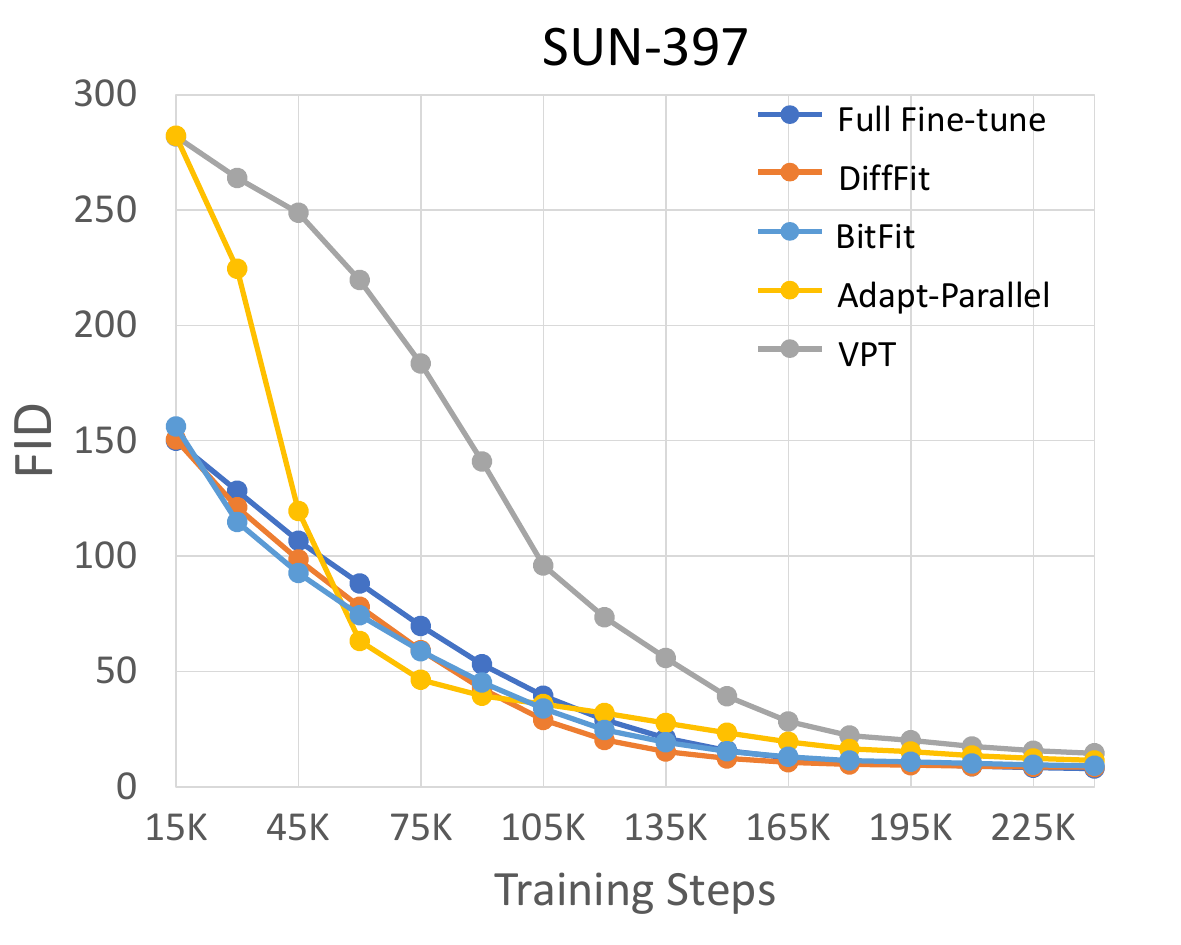}}
        \caption{FID of five methods every 15K iterations on SUN 397 dataset.}
        \label{fig:curve_sun}
    \end{minipage}
\end{figure*}

\begin{figure*}[h]
\vspace{-5mm}
\centering
\hspace{-1em}
    \begin{minipage}{0.48\linewidth}{
        \centering
        \includegraphics[width=1.0\textwidth]{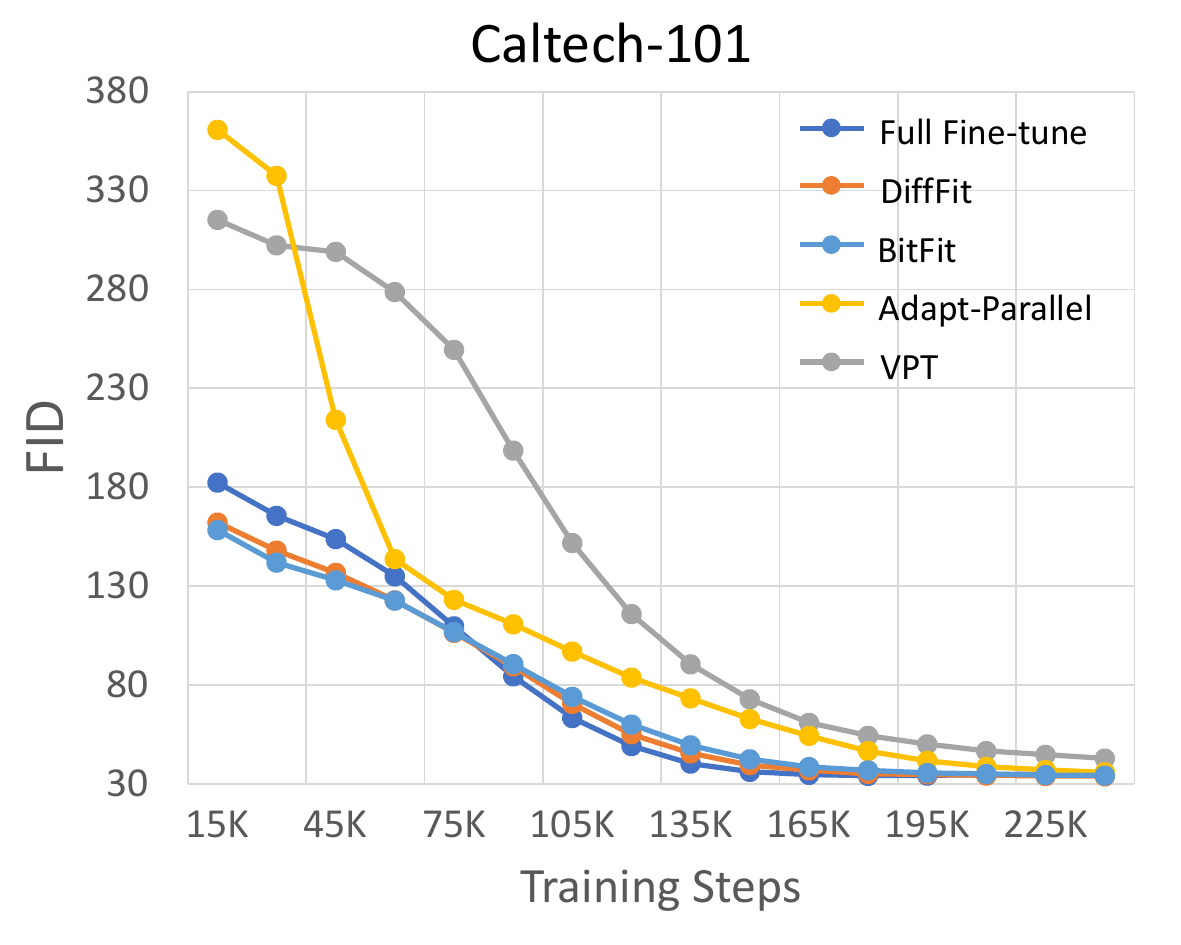}}
        \caption{FID of five methods every 15K iterations on Caltech 101 dataset.}
        \label{fig:curve_caltech}
    \end{minipage}
\hspace{2em}
    \begin{minipage}{0.48\linewidth}{
        \centering
        \includegraphics[width=1.0\textwidth]{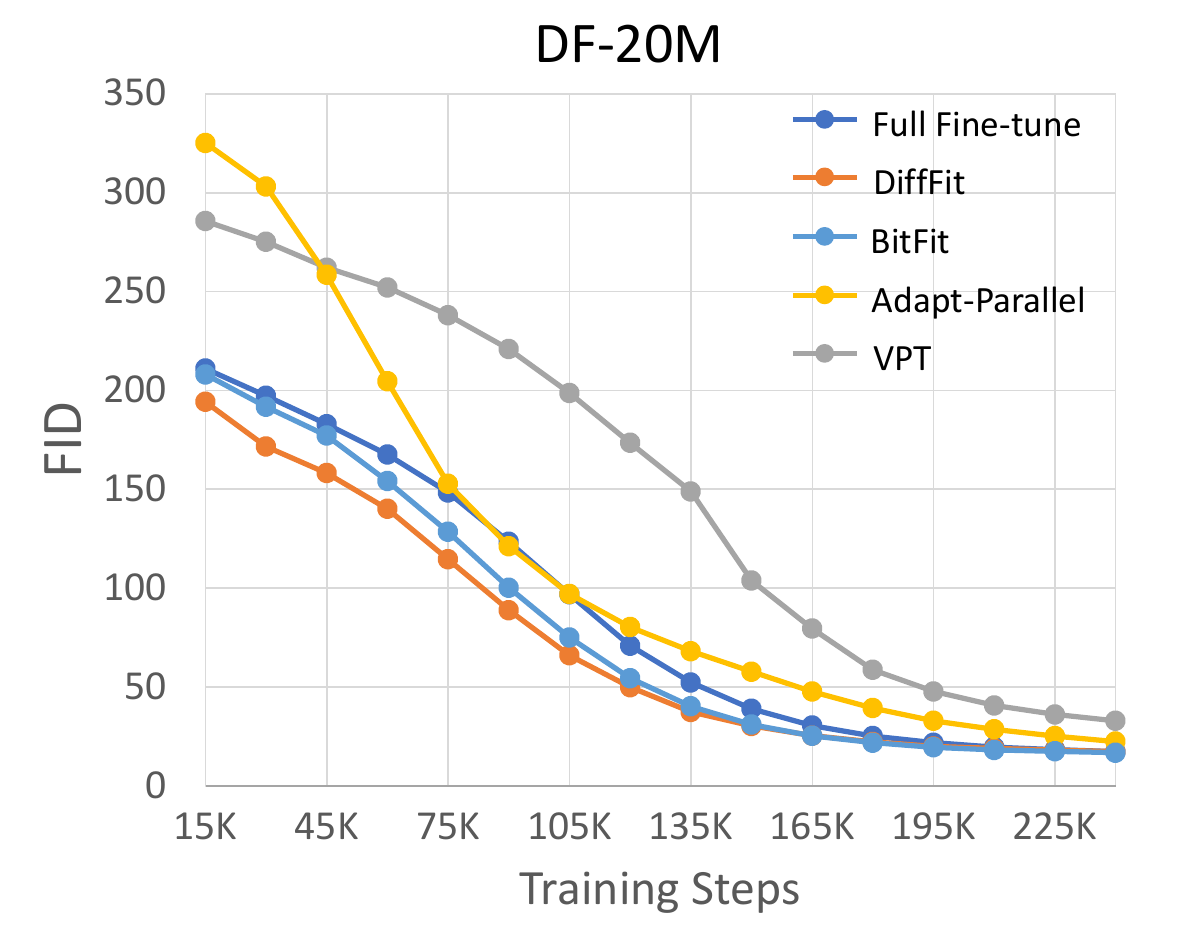}}
        \caption{FID of five methods every 15K iterations on DF20M dataset.}
        \label{fig:curve_df20m}
    \end{minipage}
\end{figure*}

\section{More Theoretical Analysis}
In this section, we will provide a more detailed analysis of the effect of scaling factors. Specifically, we will present a formal version and proof of Theorem 1, which was informally stated in the main document.

Our theoretical results are closely related to the learning of neural networks. It is worth noting that the majority of works in this area concentrate on the simplified setting where the neural network only has one non-linearly activated hidden layer with no bias terms, see, e.g.,~\cite{zhong2017recovery,ge2018learning,dugradient,zhang2019learning,gao2019learning,venturi2019spurious}. As our contributions are mainly experimental, we only provide some intuitive theoretical analysis to reveal the effect of scaling factors. To this end, we will consider the following simplified settings.

\begin{itemize}
    \item Following the approach of~\cite{kingma2021variational,chen2023analog}, we replace the noise neural network $\bepsilon_{\btheta}(\bx_t,t)$ used in the diffusion model with $\frac{\bx_t - \sqrt{\bar{\alpha}_t}\bx_{\btheta}(\bx_t,t)}{\sqrt{1-\bar{\alpha}_t}}$,\footnote{Here we assume that in the forward process, the conditional distribution $q_{0t}(\bx_t|\bx_0) = \calN(\sqrt{\bar{\alpha}_t}\bx_0,(1-\bar{\alpha}_t)\bI)$. See Section 3.1 in the main document.} where $\bx_{\btheta}(\bx_t,t)$ is a neural network that approximates the original signal. In addition, in the sampling process, we assume a single sampling step from the end time $t = E$ to the initial time $t = 0$, which gives $\bx_0 = \bx_{\btheta}(\bx_E,E)$. For brevity, we denote $\bx_{\btheta}(\cdot,E)$ by $G(\cdot)$, and we assume that $G(\bs) = f(\bW\bs + \bb)$ for all $\bs \in \bbR^D$, where $\bW \in \bbR^{D\times D}$ is a weight matrix, $\bb \in \bbR^D$ is the offset vector, and $f\,:\, \bbR \to \bbR$ represents a non-linear activation function that does not grow super-linearly ({\em cf.} Appendix~\ref{app:useful_lemmas}) and it is applied element-wise. While one-step generation may seem restrictive, recent works~\cite{liu2023flow,song2023consistency} have achieved it for the diffusion models.
    \item Suppose that we have a dataset generated from distribution $Q_0 \in \bbR^D$. Further assuming that there exist ground-truth scaling factors $\bgamma^* \in \bbR^D$ such that each entry of any data point in a relatively small dataset generated from distribution $P_0 \in \bbR^D$ can be written as $\bx^T\bgamma^*$ for some $\bx$ sampled from $Q_0$. We denote this transition from  $Q_0$ to $P_0$ as $P_0 = \bm{f}_{\bgamma^*} \# Q_0$ for brevity.
    \item Recent theoretical results for diffusion models~\cite{de2021diffusion,deconvergence,leeconvergence,chen2022sampling,chen2023score} show that  under appropriate conditions, diffusion models can generate samples that approximately follow the original data distribution. Based on these results, we assume that $G(\bepsilon) \sim \hat{Q}_0$ for $\bepsilon \sim \calN(\bm{0},\bI)$, where $\hat{Q}_0$ is close to $Q_0$ in some probability distance.  
\end{itemize}

Under the above settings, in the training process for the diffusion model using the relatively small dataset, if we only fine-tune the scaling factors, the objective function can be expressed as:
\begin{equation}\label{eq:origin_obj}
    \min_{\bgamma \in \bbR^D} \sum_{i=1}^m\|G(\bepsilon_i)^T\bgamma - \bx_i^T \bgamma^*\|_2^2,
\end{equation}
where $\bepsilon_i \sim \calN(\bm{0},\bI)$, $\bx_i \sim Q_0$ (then $\bx_i^T \bgamma^*$ corresponds to an entry of data point generated from $P_0$), and $m$ is the number of training samples. As $G(\bepsilon_i)$ is sub-Gaussian ({\em cf.} Appendix~\ref{app:general_aux}) and $G(\bepsilon_i) \sim \hat{Q}_0$ with $\hat{Q}_0$ being close to $Q_0$, we assume that $\bx_i$ can be represented as $\bx_i = G(\bepsilon_i) + \bz_i$, where each entry of $\bz_i$ is zero-mean sub-Gaussian with the sub-Gaussian norm ({\em cf.} Appendix~\ref{app:general_aux} for the definition of sub-Gaussian norm) being upper bounded by some small constant $\eta > 0$. 

Let $\ba_i = G(\bepsilon_i) = f(\bW\bepsilon_i + \bb) \in \bbR^D$. We write $\bA = [\ba_1,\ldots,\ba_m]^T \in \bbR^{m \times D}$, $\bZ = [\bz_1,\ldots,\bz_m]^T \in \bbR^{m \times D}$, and $\by = [y_1,\ldots,y_m]^T$ with $y_i =\bx_i^T \bgamma^*$. Then, we have $\by = (\bA + \bZ)\bgamma^*$ and the objective function~\eqref{eq:origin_obj} can be re-written as
\begin{equation}\label{eq:simple_obj}
    \min_{\bgamma \in \bbR^D} \|\bA \bgamma - \by\|_2^2.
\end{equation}
Let $V_{\min} := \min_{i \in \{1,\ldots,D\}} \mathrm{Var}[f(X_i)]$ with $X_i \sim \calN(b_i,\|\bw_i\|_2^2)$, where $b_i$ is the $i$-th entry of the offset vector $\bb$ and $\bw_i^T \in \bbR^{1\times D}$ is the $i$-th row of the weight matrix $\bW$. Then, we have the following theorem concerning the optimal solution to~\eqref{eq:simple_obj}. The proof of Theorem~\ref{thm:obj_min} is deferred to Section~\ref{app:main_theory}.
\begin{theorem}\label{thm:obj_min}
Under the above settings, if $\hat{\bgamma}$ is an optimal solution to~\eqref{eq:simple_obj} and $m = \Omega(D^2\log D)$, with probability $1-e^{-\Omega(D\log D)}$, it holds that
\begin{equation}
    \|\hat{\bgamma} -\bgamma^*\|_2 < \frac{4\sqrt{2}\eta}{\sqrt{V_{\min}}} \cdot \|\bgamma^*\|_2. 
\end{equation}
\end{theorem}
Note that $\eta >0$ is considered to be small and $V_{\min}$ is a fixed positive constant. In addition, the gradient descent algorithm aims to minimize~\eqref{eq:simple_obj}. Therefore, {\em Theorem~\ref{thm:obj_min} essentially says that when the number of training samples is sufficiently large, the simple gradient descent algorithm finds an estimate $\hat{\bgamma}$ that is close to $\bgamma^*$ with high probability} (in the sense that the relative distance $\|\hat{\bgamma} -\bgamma^*\|_2/\|\bgamma^*\|_2$ is small). Furthermore, {\em with this $\hat{\bgamma}$, the diffusion model generates distribution $\hat{P}_0 = \bm{f}_{\hat{\bgamma}} \# \hat{Q}_0$, which is naturally considered to be close to $P_0 = \bm{f}_{\bgamma^*} \# Q_0$ since $\hat{\bgamma}$ is close to $\bgamma^*$ and $\hat{Q}_0$ is close to $Q_0$.}

Before presenting the proof of Theorem~\ref{thm:obj_min} in Section~\ref{app:main_theory}, we provide some auxiliary results. 

\subsection{Notation}

We write $[N] = \{1, 2,\ldots, N\}$ for a positive integer $N$. The unit sphere in $ \bbR^D$ is denoted by $\calS^{D-1} := \{ \bx \in \bbR^D\,:\, \|\bx\|_2 = 1\}$. We use $\|\bX\|_{2}$ to denote the spectral norm of a matrix $\bX$. We use the symbols $C, C', c$ to denote absolute constants, whose values may differ from line to line.

\subsection{General Auxiliary Results}
\label{app:general_aux}

First, the widely-used notion of an $\epsilon$-net is presented as follows, see, e.g.,~\cite[Definition~3]{liugenerative}.

\begin{definition}
Let $(\calX, \rmd)$ be a metric space, and fix $\epsilon >0$. A subset $S \subseteq \calX$ is said be an $\epsilon$-net of $\calX$ if, for all $\bx \in \calX$, there exists some $\bs \in S$ such that $\rmd(\bs, \bx) \le \epsilon$. The minimal cardinality of an $\epsilon$-net of $\calX$ , if finite, is denoted $C(\calX,\epsilon)$ and is called the covering number of $\calX$ (at scale $\epsilon$).
\end{definition}

The following lemma provides a useful bound for the covering number of the unit sphere. 
\begin{lemma}{\em \hspace{1sp}\cite[Lemma~5.2]{vershynin2010introduction}}\label{lem:unit_sphere_cov}
 The unit Euclidean sphere $\calS^{D-1}$ equipped
with the Euclidean metric satisfies for every $\epsilon > 0$ that
\begin{equation}
 \calC(\calS^{D-1},\epsilon) \le \left(1+\frac{2}{\epsilon}\right)^{D}.
\end{equation}
\end{lemma}

In addition, we have the following lemma concerning the spectral norm of a matrix.
\begin{lemma}{\em \hspace{1sp}\cite[Lemma~5.3]{vershynin2010introduction}}
\label{lem:spectral_norm}
Let $\bA$ be an $m \times D$ matrix, and let $\calN_\epsilon$ be an $\epsilon$-net of $\calS^{D-1}$ for some $\epsilon \in (0,1)$. Then
\begin{equation}
    \|\bA\|_{2} \le (1-\epsilon)^{-1}\max_{\bx \in \calN_{\epsilon}} \|\bA\bx\|_2.
\end{equation}
\end{lemma}

Standard definitions for sub-Gaussian and sub-exponential random variables are presented as follows. 
\begin{definition} \label{def:subg}
 A random variable $X$ is said to be sub-Gaussian if there exists a positive constant $C$ such that $\left(\mathbb{E}\left[|X|^{p}\right]\right)^{1/p} \leq C  \sqrt{p}$ for all $p\geq 1$, and the corresponding sub-Gaussian norm is defined as $\|X\|_{\psi_2}:=\sup_{p\ge 1} p^{-1/2}\left(\mathbb{E}\left[|X|^{p}\right]\right)^{1/p}$. 
\end{definition}

\begin{definition}
 A random variable $X$ is said to be sub-exponential if there exists a positive constant $C$ such that $\left(\bbE\left[|X|^p\right]\right)^{1/p} \le C p$ for all $p \ge 1$, and the corresponding sub-exponential norm is defined as $\|X\|_{\psi_1} := \sup_{p \ge 1} p^{-1} \left(\bbE\left[|X|^p\right]\right)^{1/p}$.
\end{definition}

The following lemma concerns the relation between sub-Gaussian and sub-exponential random variables. 
\begin{lemma}{\em \hspace{1sp}\cite[Lemma~5.14]{vershynin2010introduction}}\label{lem:prod_subGs}
 A random variable $X$ is sub-Gaussian if and only if $X^2$ is sub-exponential. Moreover,
  \begin{equation}
   \|X\|_{\psi_2}^2 \le \|X\|_{\psi_1}^2 \le 2\|X\|_{\psi_2}^2.
  \end{equation}
\end{lemma}

The following lemma provides a useful concentration inequality for the sum of independent sub-exponential random variables.
\begin{lemma}{\em \hspace{1sp}\cite[Proposition~5.16]{vershynin2010introduction}}\label{lem:large_dev}
Let $X_{1}, \ldots , X_{N}$ be independent zero-mean sub-exponential random variables, and $K = \max_{i} \|X_{i} \|_{\psi_{1}}$. Then for every $\balpha = [\alpha_1,\ldots,\alpha_N]^T \in \bbR^N$ and $\epsilon \geq 0$, it holds that
\begin{align}
 & \mathbb{P}\bigg( \Big|\sum_{i=1}^{N}\alpha_i X_{i}\Big|\ge \epsilon\bigg)  \leq 2  \exp \left(-c \cdot \mathrm{min}\Big(\frac{\epsilon^{2}}{K^{2}\|\balpha\|_2^2},\frac{\epsilon}{K\|\balpha\|_\infty}\Big)\right), \label{eq:subexp}
\end{align}
where $c > 0$ is an absolute constant.  In particular, with $\balpha = \big[ \frac{1}{N},\dotsc,\frac{1}{N} \big]^T$, we have
\begin{align}
    & \mathbb{P}\bigg( \Big| \frac{1}{N} \sum_{i=1}^{N} X_{i}\Big|\ge \epsilon\bigg)  \leq 2  \exp \left(-c \cdot \mathrm{min}\Big(\frac{N \epsilon^{2}}{K^{2}},\frac{N \epsilon}{K}\Big)\right). \label{eq:subexp2}
\end{align}
\end{lemma}

\subsection{Useful Lemmas}
\label{app:useful_lemmas}
Throughout the following, we use $f(\cdot)$ to denote some non-linear activation function that does not grow faster than linear, i.e., there exist scalars $a$ and $b$ such that $f(x) \le a|x| +b$ for all $x\in \bbR$. Then, if $X \sim \calN(\mu,\sigma^2)$ is a random Gaussian variable, $f(X)$ will be sub-Gaussian~\cite{liu2020generalized}. Note that the condition that $f(\cdot)$ does not grow super-linearly is satisfied by popular activation functions such as ReLU, Sigmoid, and Hyperbolic tangent function. 

\begin{lemma}\label{lem:exp_bds}
    For $i \in [N]$, let $X_i \sim \calN(\mu_i, \sigma_i^2)$ be a random Gaussian variable. Then for every $\bgamma = [\gamma_1,\ldots,\gamma_N]^T \in \bbR^N$, we have 
    \begin{equation}
         V_{\min}  \|\bgamma\|_2^2 \le \bbE\left[\left(\sum_{i=1}^N \gamma_i f(X_i)\right)^2\right] \le (N+1)U_{\max}\|\bgamma\|_2^2,
    \end{equation}
    where $V_{\min} = \min_{i\in [N]} \mathrm{Var}\left[f(X_i)\right]$ and $U_{\max} = \max_{i\in [N]} \bbE\big[f(X_i)^2\big]$ are dependent on $\{\mu_i, \sigma_i^2\}_{i \in [N]}$ and the non-linear activation function $f(\cdot)$. 
\end{lemma}
\begin{proof}
We have 
\begin{align}
    \bbE\left[\left(\sum_{i=1}^N \gamma_i f(X_i)\right)^2\right] & = \sum_{i=1}^N \gamma^2_i \bbE\left[f(X_i)^2\right] + \sum_{i \ne j} \gamma_i \gamma_j \bbE[f(X_i)]\cdot \bbE[f(X_j)] \\
    & =  \sum_{i=1}^N \gamma^2_i\mathrm{Var}\left[f(X_i)\right] + \left(\sum_{i=1}^N \gamma_i \bbE[f(X_i)]\right)^2 \label{eq:lem_firstIneq}\\
    & \ge \sum_{i=1}^N \gamma^2_i\mathrm{Var}\left[f(X_i)\right] \\
    & \ge V_{\min} \|\bgamma\|_2^2.
\end{align}
In addition, from~\eqref{eq:lem_firstIneq}, by the Cauchy-Schwarz Inequality, we obtain
\begin{align}
     \bbE\left[\left(\sum_{i=1}^N \gamma_i f(X_i)\right)^2\right] & \le U_{\max} \|\bgamma\|_2^2 + \left(\sum_{i=1}^N \bbE[f(X_i)]^2\right)\|\bgamma\|_2^2 \\
     & \le (N+1)U_{\max} \|\bgamma\|_2^2.
\end{align}
This completes the proof. 
\end{proof}

\begin{lemma}\label{lem:spectral_norm_bound}
    Let $\bE \in \bbR^{m \times D}$ be a standard Gaussian matrix, i.e., each entry of $\bE$ is sampled from standard Gaussian distribution, and let $\bW \in \bbR^{D \times D}$ be a fixed matrix that has no zero rows. In addition, for a fixed vector $\bb \in \bbR^D$, let $\bB = [\bb,\bb,\ldots,\bb]^T \in\bbR^{m \times D}$. Then, when $m = \Omega(D)$, with probability $1-e^{-\Omega(m)}$, it holds that  
    \begin{equation}
        \frac{1}{\sqrt{m}}\|f(\bE\bW^T + \bB)\|_{2} \le C\sqrt{D},
    \end{equation}
    where $C$ is an absolute constant and the non-linear activation function $f(\cdot)$ is applied element-wise. 
\end{lemma}
\begin{proof}
    Let $\bw_i^T \in \bbR^{1 \times D}$ be the $i$-th row of $\bW$. By the assumption that $\bW$ has no zero rows, we have $\|\bw_i\|_2 > 0$ for all $i \in [N]$. Let $\bH = \bE \bW^T + \bB \in \bbR^{m \times D}$. Then the $(i,j)$-th entry of $\bH$, denoted $h_{ij}$, follows the $\calN\big(b_j,\|\bw_j\|_2^2\big)$ distribution. For any $\bgamma \in \calS^{D-1}$ and any fixed $i$, from Lemma~\ref{lem:prod_subGs}, we obtain that $\big(\sum_{j=1}^D f(h_{ij})\gamma_j\big)^2$ is sub-exponential with the sub-exponential norm being upper bounded by $C$, where $C$ is some absolute constant. Let $E_{\bgamma}$ be the expectation of $\big(\sum_{j=1}^D f(h_{ij})\gamma_j\big)^2$. From Lemma~\ref{lem:exp_bds}, we obtain that it holds uniformly for all $\bgamma \in \calS^{D-1}$ that
    \begin{equation}\label{eq:unif_Egamma_bd}
        E_{\bgamma} \le C' D,
    \end{equation}
    where $C'$ is an absolute constant. In addition, from Lemma~\ref{lem:large_dev}, for any $\epsilon \in (0,1)$, we obtain
    \begin{align}\label{eq:prob_bd_single}
        \bbP\left(\left|\frac{1}{m}\sum_{i=1}^m \left(\sum_{j=1}^D f(h_{ij})\gamma_j\right)^2 - E_{\bgamma}\right| \ge \epsilon\right) \le 2\exp(-\Omega(m\epsilon^2)).
    \end{align}
    From Lemma~\ref{lem:unit_sphere_cov}, there exists an $\epsilon$-net $\calN_{\epsilon}$ of $\calS^{D-1}$ with $\big|\calN_{\epsilon}\big|\le (1+2/\epsilon)^D$. Taking a union bound over $\calN_{\epsilon}$, we have that when $m = \Omega\big(\frac{D}{\epsilon^2}\log\frac{1}{\epsilon}\big)$, with probability $1-e^{-\Omega(m\epsilon^2)}$, it holds for all $\bgamma \in \calN_{\epsilon}$ that 
\begin{align}\label{eq:imp_twosided_bds}
   E_{\bgamma} -\epsilon \le  \frac{1}{m}\|f(\bE\bW^T + \bB)\bgamma\|_2^2 = \frac{1}{m}\sum_{i=1}^m \left(\sum_{j=1}^D f(h_{ij})\gamma_j\right)^2 \le E_{\bgamma} +\epsilon.
\end{align}
   Then, since~\eqref{eq:unif_Egamma_bd} holds uniformly for all $\bgamma \in \calS^{D-1}$, setting $\epsilon = \frac{1}{2}$, Lemma~\ref{lem:spectral_norm} implies that when $m = \Omega(D)$, with probability $1-e^{-\Omega(m)}$, 
    \begin{align}
        \frac{1}{\sqrt{m}}\|f(\bE\bW^T + \bB)\|_{2} &\le 2 \max_{\bgamma \in \calN_{1/2}} \frac{1}{\sqrt{m}}\|f(\bE\bW^T + \bB)\bgamma\|_2 \\
        & \le 2 \max_{\bgamma \in \calN_{1/2}} \sqrt{E_{\bgamma}+\frac{1}{2}} \\
        & \le C \sqrt{D}. 
    \end{align}
\end{proof}

\begin{lemma}\label{lem:restr_eigen}
   Let us use the same notation as in Lemma~\ref{lem:spectral_norm_bound}. When $m = \Omega(D^2 \log D)$, with probability $1-e^{-\Omega(D\log D)}$, it holds for {\em all} $\bgamma \in \bbR^D$ that
    \begin{equation}
        \frac{1}{\sqrt{m}}\|f(\bE\bW^T + \bB)\bgamma\|_2 >\sqrt{\frac{V_{\min}}{2}}\|\bgamma\|_2,
    \end{equation}
    where $V_{\min} = \min_{i \in [D]} \mathrm{Var}[f(X_i)]$ with $X_i \sim \calN(b_i,\|\bw_i\|_2^2)$. 
\end{lemma}
\begin{proof}
It suffices to consider the case that $\bgamma \in \calS^{D-1}$. From~\eqref{eq:imp_twosided_bds}, we have that when $m = \Omega\big(\frac{D}{\epsilon^2}\log\frac{1}{\epsilon}\big)$, with probability $1-e^{-\Omega(m\epsilon^2)}$, it holds for all $\bgamma \in \calN_{\epsilon}$ that 
    \begin{equation}\label{eq:imp_lower_bds}
        \frac{1}{m}\|f(\bE\bW^T+\bB)\bgamma\|_2^2 \ge  E_{\bgamma} -\epsilon. 
    \end{equation}
    In addition, from Lemma~\ref{lem:exp_bds}, we have that it hold uniformly for all $\bgamma \in \calS^{D-1}$ that 
    \begin{equation}\label{eq:E_gamma_lb}
        E_{\bgamma} \ge V_{\min}.
    \end{equation}
    Then, if setting $\epsilon  = \frac{c}{\sqrt{D}}$ for some small absolute constant $c$, we have that when $m = \Omega\big(D^2 \log D\big)$, with probability $1-e^{-\Omega(D\log D)}$, for all $\bgamma \in \calN_{\epsilon}$, 
    \begin{equation}
        \frac{1}{\sqrt{m}}\|f(\bE\bW^T+\bB)\bgamma\|_2 \ge \sqrt{E_{\bgamma} -\epsilon}.
    \end{equation}
    For any $\bgamma \in \calS^{D-1}$, there exists an $\bs \in \calN_{\epsilon}$ such that $\|\bs -\bgamma\|_2 \le \epsilon$, and thus
    \begin{align}
        \frac{1}{\sqrt{m}}\|f(\bE\bW^T + \bB)\bgamma\|_2 &  \ge \frac{1}{\sqrt{m}}\|f(\bE\bW^T + \bB)\bs\|_2 - \frac{1}{\sqrt{m}}\|f(\bE\bW^T+\bB)(\bgamma-\bs)\|_2 \\
        & \ge  \sqrt{E_{\bgamma} -\epsilon} -  \frac{1}{\sqrt{m}}\|f(\bE\bW^T+\bB)\|_{2} \cdot \epsilon\\
        & \ge  \sqrt{E_{\bgamma} -\epsilon} - cC,\label{eq:last_cC}
    \end{align}
    where~\eqref{eq:last_cC} follows from Lemma~\ref{lem:spectral_norm_bound} and the setting $\epsilon = \frac{c}{\sqrt{D}}$. Then, if $c$ is set to be  sufficiently small, from~\eqref{eq:E_gamma_lb}, we have for all $\bgamma \in \calS^{D-1}$ that
    \begin{equation}
       \frac{1}{\sqrt{m}}\|f(\bE\bW^T+\bB)\bgamma\|_2 > \sqrt{\frac{V_{\min}}{2}}.
    \end{equation}
\end{proof}

\begin{lemma}\label{lem:last_bd}
    Let $\bZ \in \bbR^{m \times D}$ be a matrix with independent zero-mean sub-Gaussian entries. Suppose that the sub-Gaussian norm of all entries of $\bZ$ is upper bounded by some $\eta > 0$. Then, for any $\bgamma \in \bbR^D$, we have that with probability $1-e^{-\Omega(m)}$, 
    \begin{equation}
        \frac{1}{\sqrt{m}}\|\bZ\bgamma\|_2 \le 2 \eta \|\bgamma\|_2.
    \end{equation}
\end{lemma}
\begin{proof}
    Let $\bz_i^T \in \bbR^{1\times D}$ be the $i$-th row of $\bZ \in \bbR^{m\times D}$. From the definition of sub-Gaussian norm, we have that if a random variable $X$ is zero-mean sub-Gaussian with $\|X\|_{\psi_2} \le \eta$, then
    \begin{equation}
        \eta = \sup_{p \ge 1} p^{-1/2} (\bbE[|X|^p])^{1/p} \ge  2^{-1/2} (\bbE[|X|^2])^{1/2}, 
    \end{equation}
    which implies
    \begin{equation}
        \bbE[X^2] \le 2\eta^2. \label{eq:simple_def_bd}
    \end{equation}
    Then, we have 
    \begin{equation}
        \bbE[(\bz_i^T\bgamma)^2] \le 2\eta^2 \|\bgamma\|_2^2. 
    \end{equation}
    In addition, from Lemma~\ref{lem:prod_subGs}, we also have that $(\bz_i^T\bgamma)^2$ is sub-exponential with the sub-exponential norm being upper bounded by $C\eta^2\|\bgamma\|_2^2$. Then, from Lemma~\ref{lem:large_dev}, we obtain that for any $\epsilon \in (0,1)$, with probability $1-e^{-\Omega(m \epsilon^2)}$, 
    \begin{equation}
        \left|\frac{1}{m}\sum_{i=1}^m \left((\bz_i^T\bgamma)^2 - \bbE\left[(\bz_i^T\bgamma)^2\right]\right)\right| \le C\epsilon \eta^2\|\bgamma\|_2^2,
    \end{equation}
    which implies 
    \begin{equation}
       \frac{1}{m}\|\bZ\bgamma\|_2^2  = \frac{1}{m}\sum_{i=1}^m (\bz_i^T\bgamma)^2 \le \frac{1}{m} \sum_{i=1}^m \bbE\left[(\bz_i^T\bgamma)^2\right] + C\epsilon \eta^2\|\bgamma\|_2^2 \le 2\eta^2 \|\bgamma\|_2^2 +  C\epsilon \eta^2\|\bgamma\|_2^2,
    \end{equation}
    where the last inequality follows from~\eqref{eq:simple_def_bd}. Setting $\epsilon = \min\{\frac{2}{C},1\}$, we obtain the desired result. 
\end{proof}

\subsection{Proof of Theorem~\ref{thm:obj_min}} 
\label{app:main_theory}

    Let $\bE = [\bepsilon_1,\ldots,\bepsilon_m]^T \in \bbR^{m \times D}$. Then $\bA$ can be written as $\bA = f(\bE \bW^T + \bB)$, where $\bB = [\bb,\ldots,\bb]^T \in \bbR^{m \times D}$. Then, we have 
    \begin{align}
         \sqrt{m}\|\hat{\bgamma} -\bgamma^*\|_2 & <  \sqrt{\frac{2}{V_{\min}}} \left\|f(\bE \bW^T + \bB) (\hat{\bgamma} -\bgamma^*)\right\|_2 \label{eq:thm2_1}\\
         &  = \sqrt{\frac{2}{V_{\min}}} \left\| \bA (\hat{\bgamma} -\bgamma^*)\right\|_2 \label{eq:thm2_2}\\
         & \le  \sqrt{\frac{2}{V_{\min}}} \left(\left\| \bA \hat{\bgamma} -\by\right\|_2 + \left\|\by - \bA\bgamma^*\right\|_2\right)\label{eq:thm2_3} \\
         & \le\sqrt{\frac{8}{V_{\min}}} \left\|\by - \bA\bgamma^*\right\|_2 \label{eq:thm2_4}\\
         & = \sqrt{\frac{8}{V_{\min}}} \|\bZ\bgamma^*\|_2 \label{eq:thm2_5}\\
         & \le \frac{4\sqrt{2m}\eta}{\sqrt{V_{\min}}} \cdot \|\bgamma^*\|_2,\label{eq:thm2_6}
    \end{align}
    where~\eqref{eq:thm2_1} follows from Lemma~\ref{lem:restr_eigen},~\eqref{eq:thm2_3} follows from the triangle inequality,~\eqref{eq:thm2_4} follows from the condition that $\hat{\bgamma}$ minimizes~\eqref{eq:simple_obj},~\eqref{eq:thm2_5} follows from $\by = (\bA + \bZ) \bgamma^*$. Finally, we use Lemma~\ref{lem:last_bd} to obtain~\eqref{eq:thm2_6}.

\section{More Visualization Results}
In this section, we present additional samples generated by the DiT-XL/2 + DiffFit fine-tuning. Specifically, we demonstrate its efficacy by generating high-quality images with a resolution of 512$\times$512 on the ImageNet dataset, as illustrated in Figures~\ref{fig:imagenet512-1}, \ref{fig:imagenet512-2}, and \ref{fig:imagenet512-3}. Additionally, we showcase the model's capability on 8 downstream tasks by displaying its image generation results with a resolution of 256$\times$256, as depicted in Figures~\ref{fig:vis_food} through \ref{fig:vis_sun}.

\begin{figure*}[t!]
    \centering
    \includegraphics[width=1.0\textwidth]{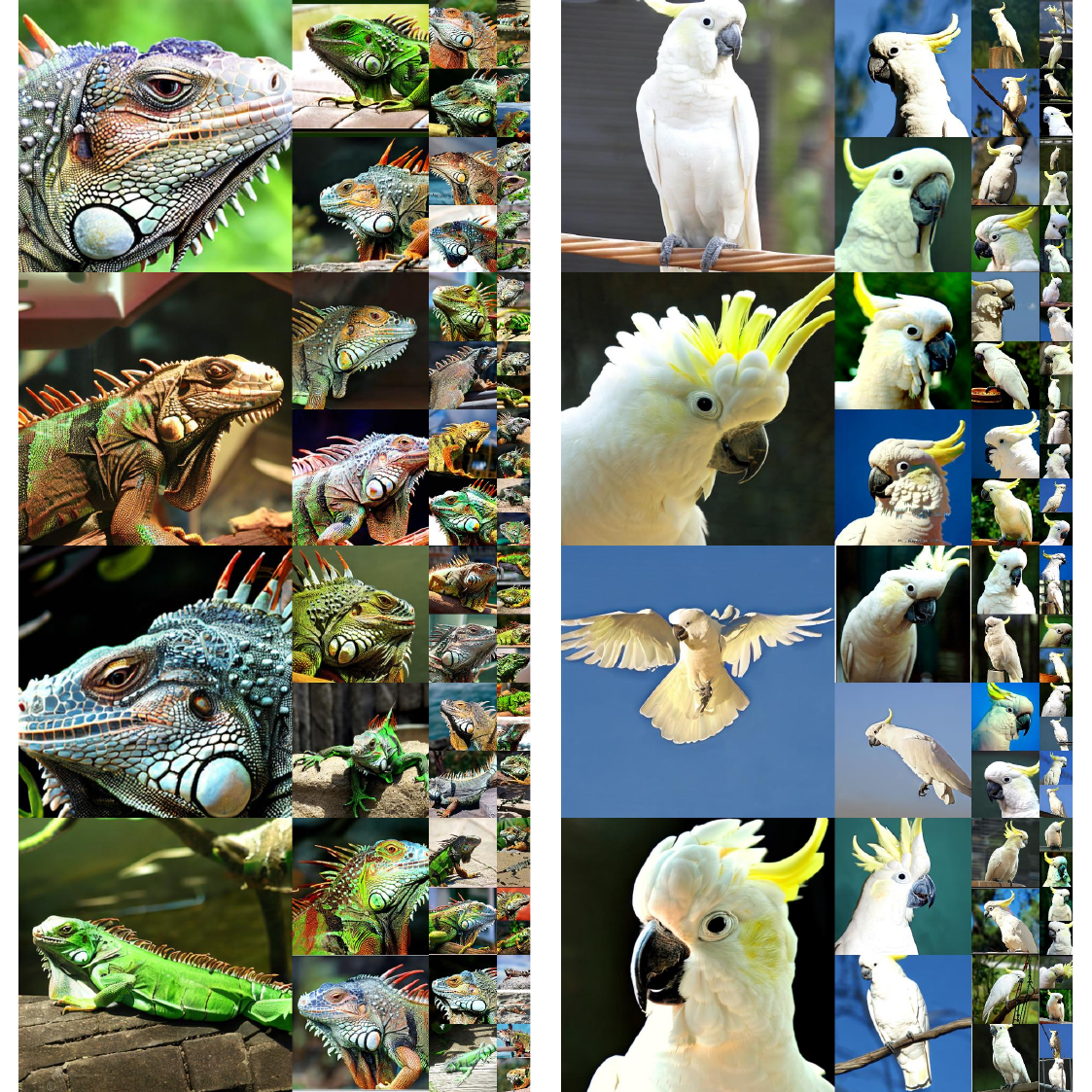}
    \caption{
    Visualization of DiffFit on ImageNet 512$\times$512. 
    Classifier-free guidance scale = 4.0, sampling steps = 250.
    }
    \label{fig:imagenet512-1}
\end{figure*}

\begin{figure*}[t!]
    \centering
    \includegraphics[width=1.0\textwidth]{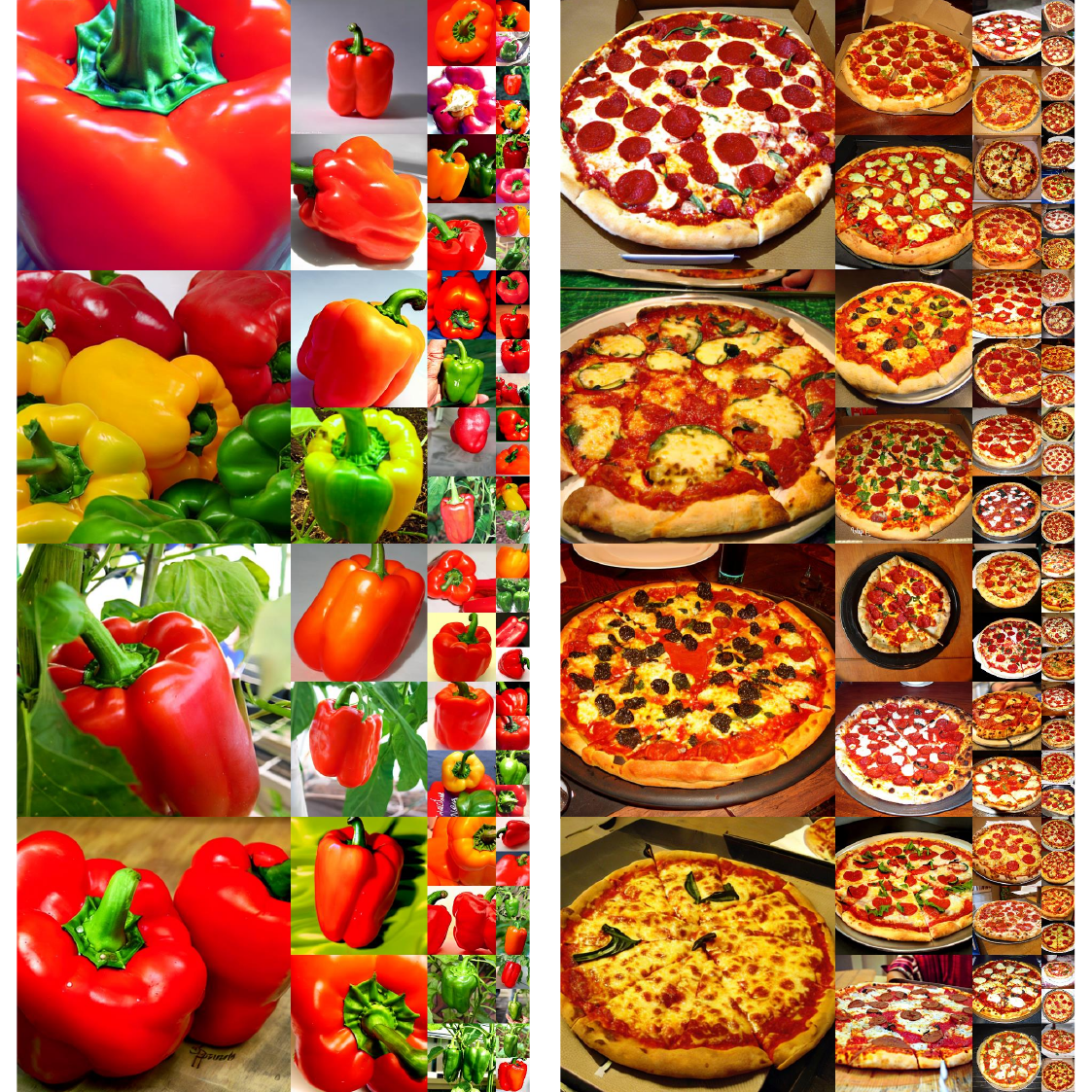}
    \caption{
    Visualization of DiffFit on ImageNet 512$\times$512. 
    Classifier-free guidance scale = 4.0, sampling steps = 250.
    }
    \label{fig:imagenet512-2}
\end{figure*}

\begin{figure*}[t!]
    \centering
    \includegraphics[width=1.0\textwidth]{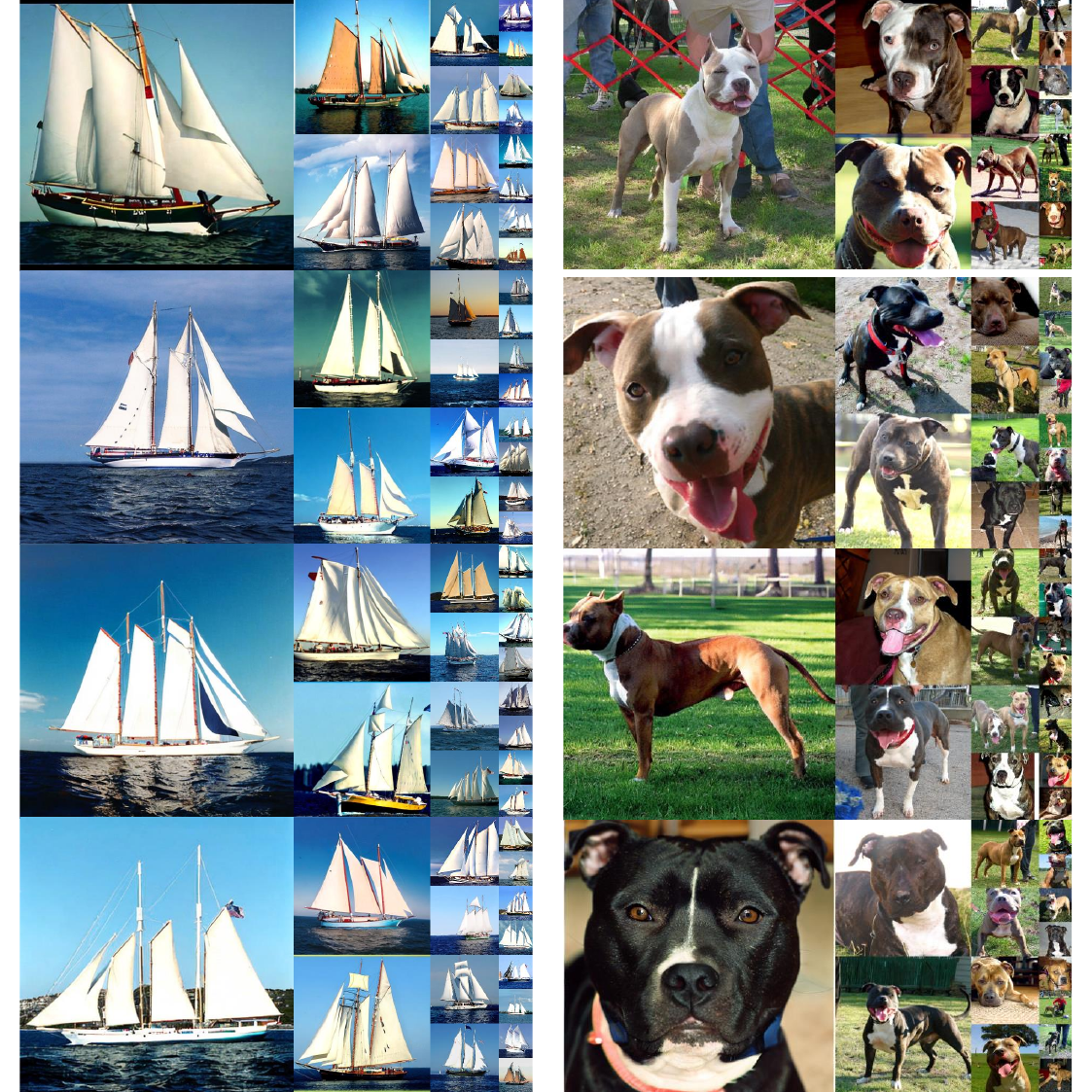}
    \caption{
    Visualization of DiffFit on ImageNet 512$\times$512. 
    Classifier-free guidance scale = 4.0, sampling steps = 250.
    }
    \label{fig:imagenet512-3}
\end{figure*}

\begin{figure*}[t]
\vspace{-5mm}
\centering
\hspace{-1em}
    \begin{minipage}{0.48\linewidth}{
        \centering
        \includegraphics[width=1.0\textwidth]{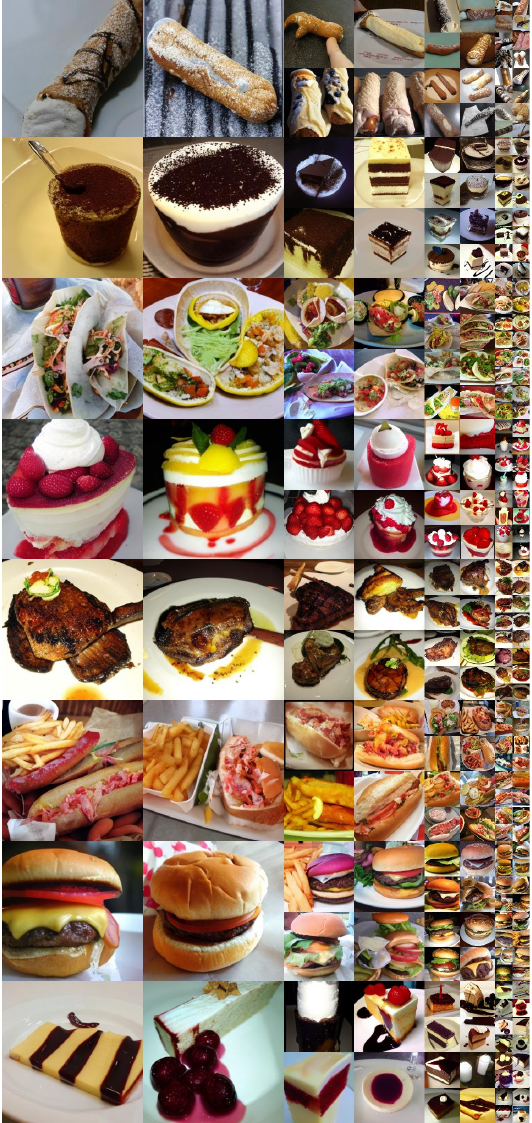}
    \caption{
    Visualization of DiffFit on Food 101. \\
    Classifier-free guidance scale = 4.0, sampling steps = 250.
    }
    \label{fig:vis_food}}
    \end{minipage}
\hspace{2em}
    \begin{minipage}{0.48\linewidth}{
        \centering
        \includegraphics[width=1.0\textwidth]{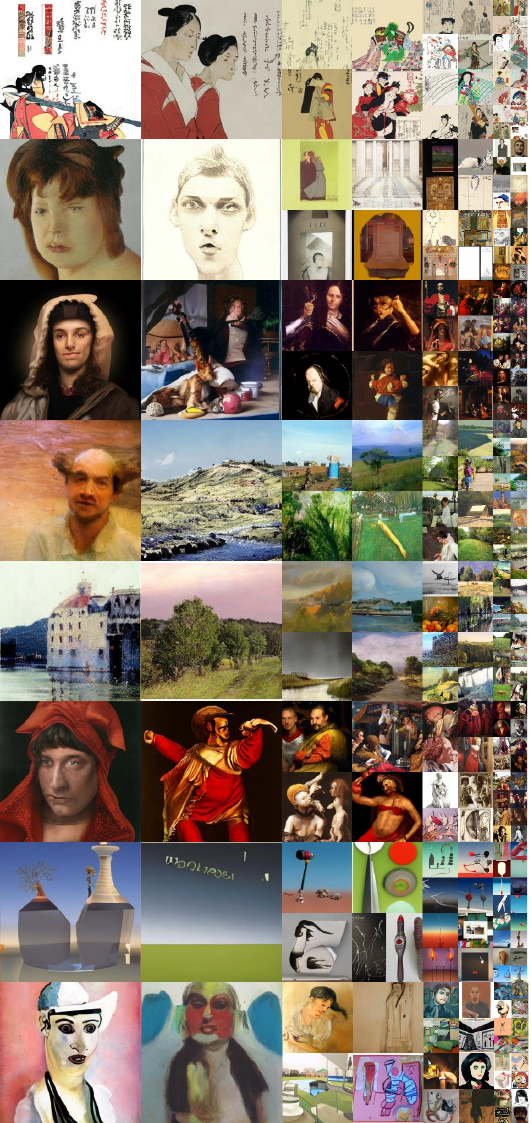}
    \caption{
    Visualization of DiffFit on ArtBench 10. \\
    Classifier-free guidance scale = 4.0, sampling steps = 250.
    }
    \label{fig:vis_art}}
    \end{minipage}
\end{figure*}

\begin{figure*}[t]
\vspace{-5mm}
\centering
\hspace{-1em}
    \begin{minipage}{0.48\linewidth}{
        \centering
        \includegraphics[width=1.0\textwidth]{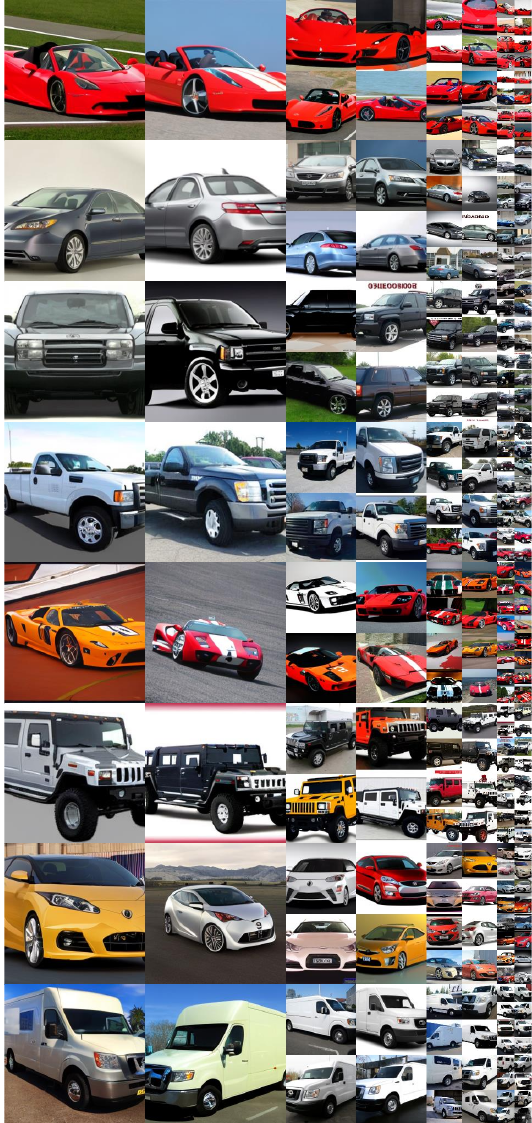}
    \caption{
    Visualization of DiffFit on Stanford Cars. \\
    Classifier-free guidance scale = 4.0, sampling steps = 250.
    }
    \label{fig:vis_car}}
    \end{minipage}
\hspace{2em}
    \begin{minipage}{0.48\linewidth}{
        \centering
        \includegraphics[width=1.0\textwidth]{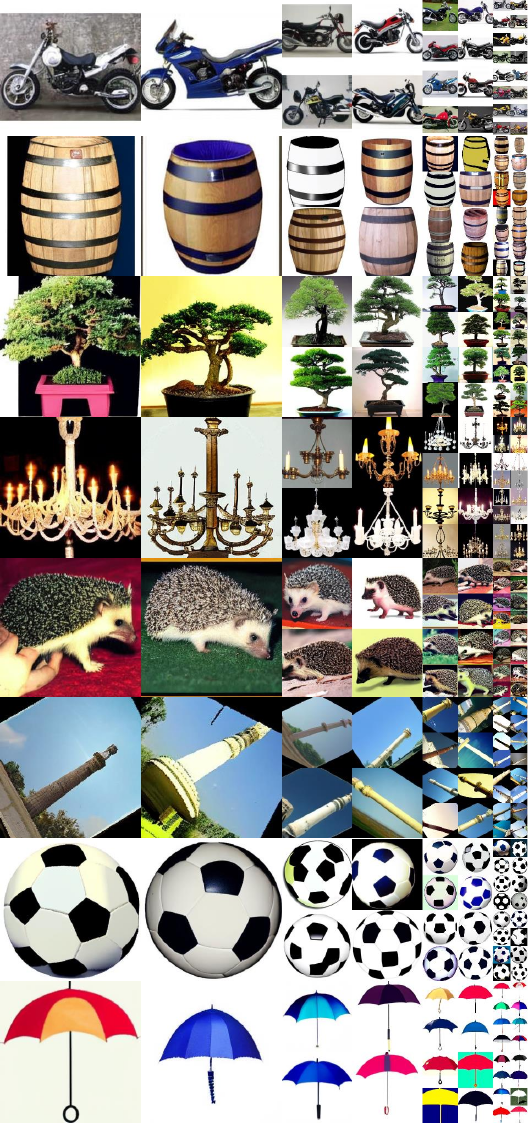}
    \caption{
    Visualization of DiffFit on Caltech 101. \\
    Classifier-free guidance scale = 4.0, sampling steps = 250.
    }
    \label{fig:vis_caltech}}
    \end{minipage}
\end{figure*}

\begin{figure*}[t]
\vspace{-5mm}
\centering
\hspace{-1em}
    \begin{minipage}{0.48\linewidth}{
        \centering
        \includegraphics[width=1.0\textwidth]{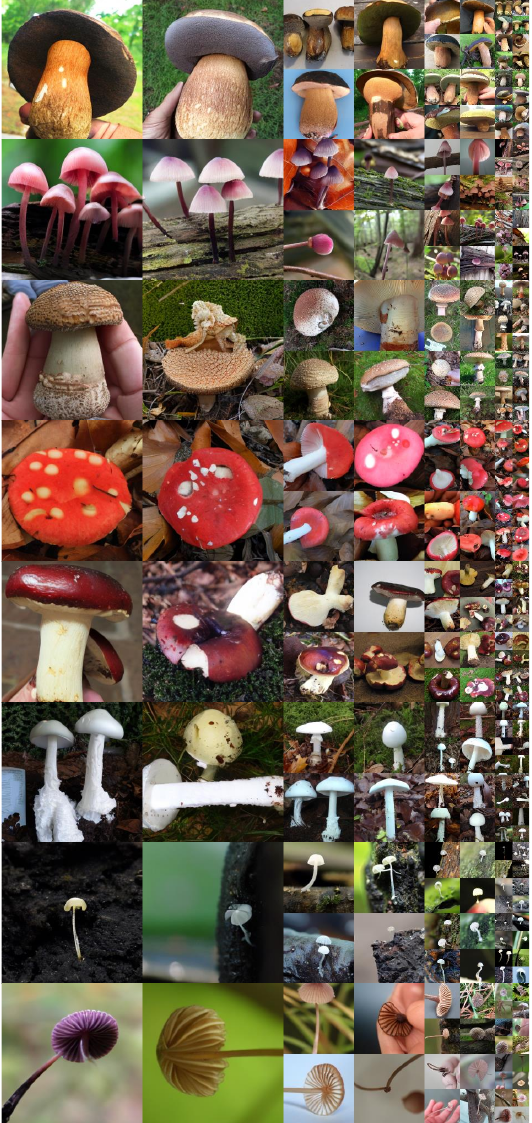}
    \caption{
    Visualization of DiffFit on DF20M. \\
    Classifier-free guidance scale = 4.0, sampling steps = 250.
    }
    \label{fig:vis_df20m}}
    \end{minipage}
\hspace{2em}
    \begin{minipage}{0.48\linewidth}{
        \centering
        \includegraphics[width=1.0\textwidth]{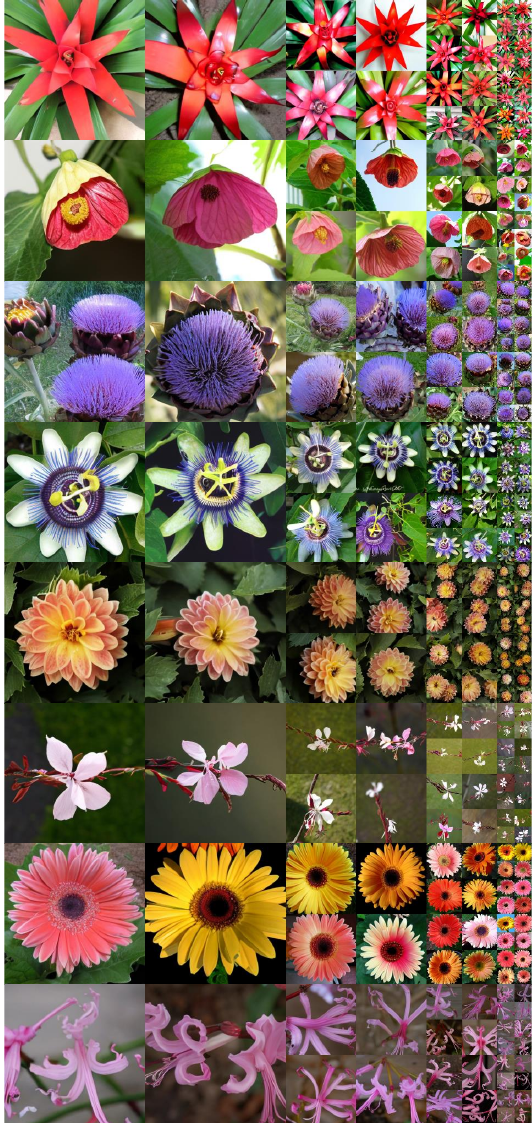}
    \caption{
    Visualization of DiffFit on Flowers 102. \\
    Classifier-free guidance scale = 4.0, sampling steps = 250.
    }
    \label{fig:vis_flower}}
    \end{minipage}
\end{figure*}

\begin{figure*}[t]
\vspace{-5mm}
\centering
\hspace{-1em}
    \begin{minipage}{0.48\linewidth}{
        \centering
        \includegraphics[width=1.0\textwidth]{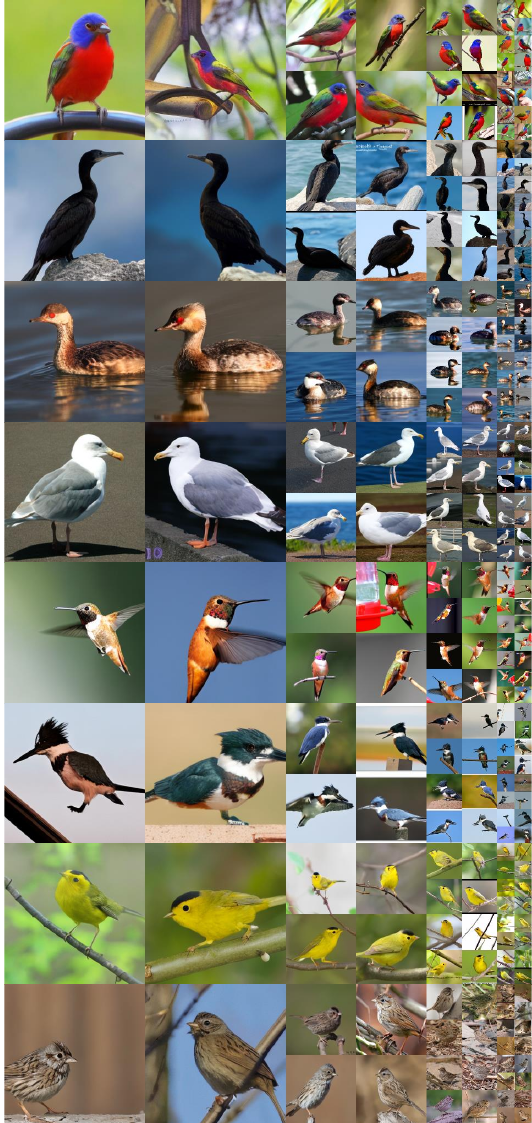}
    \caption{
    Visualization of DiffFit on CUB-200-2011. \\
    Classifier-free guidance scale = 4.0, sampling steps = 250.
    }
    \label{fig:vis_cub}}
    \end{minipage}
\hspace{2em}
    \begin{minipage}{0.48\linewidth}{
        \centering
       \includegraphics[width=1.0\textwidth]{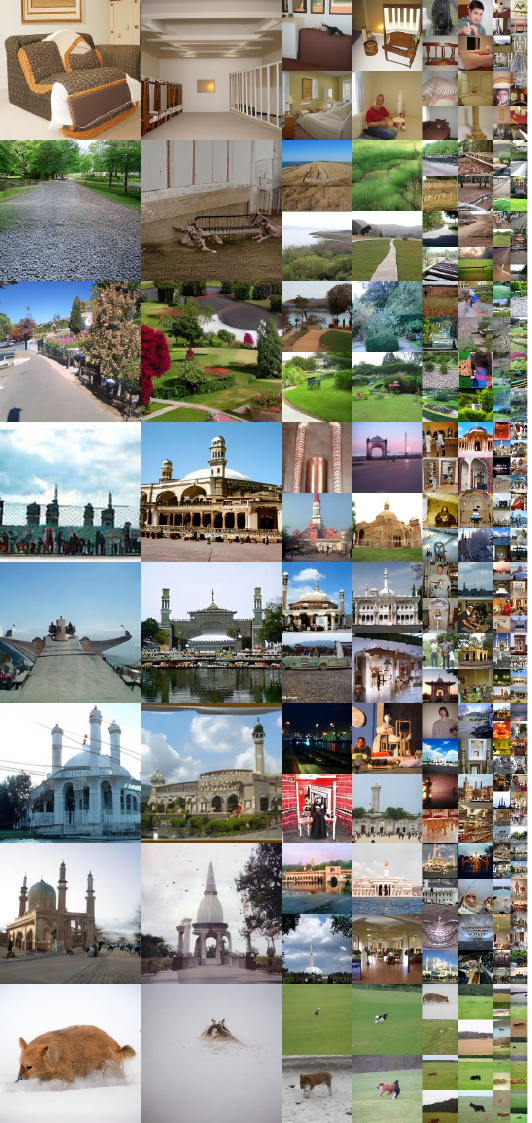}
    \caption{
    Visualization of DiffFit on SUN 397. \\
    Classifier-free guidance scale = 4.0, sampling steps = 250.
    }
    \label{fig:vis_sun}}
    \end{minipage}
\end{figure*}

\clearpage
\newpage

{\small
\bibliographystyle{ieee_fullname}
\bibliography{egbib}
}

\end{document}